\documentclass[preprint,twoside,11pt]{article}
\usepackage{jmlr2e}

\usepackage{amssymb}
\usepackage{mathtools}
\usepackage{caption, subcaption, graphicx}
\usepackage[dvipsnames]{xcolor}
\usepackage{algorithm}
\usepackage{algpseudocode} 
\usepackage{enumitem}
\usepackage{tikz}
\usepackage{adjustbox}
\usepackage{bm}
\usetikzlibrary{matrix,positioning,arrows.meta,calc,shadows,shapes,fit}
\usepackage{lineno}

\usepackage{hyperref}
\hypersetup{
  colorlinks   = true, 
  urlcolor     = blue, 
  linkcolor    = blue, 
  citecolor   = blue 
}
\usepackage{url}
\usepackage{cleveref}

\renewcommand{\eqref}[1]{\textup{(\ref{#1})}}

\newcommand{\LL}{\mathcal{L}}
\newcommand{\F}{\mathcal{F}}
\renewcommand{\P}{{\mathcal{P}}}

\newcommand{\mup}{\textup{m}}

\newcommand{\neweps}{\nu}

\newcommand{\RR}{\mathbb{R}}

\newcommand{\w}{\mathbf{w}}

\newcommand{\unif}{\textup{Unif}}

\newcommand{\ve}{\varepsilon}
\newcommand{\W}{\mathbb{W}}

\newcommand{\D}{D^\wedge}

\newtheorem{assumption}[theorem]{Assumption}

\ShortHeadings{Two-Time-Scale Learning Dynamics}{Borghi, Im and Pareschi}
\firstpageno{1}

\begin{document}


\title{Two-Time-Scale Learning Dynamics:\\ A Population View of Neural Network Training}

\author{\name Giacomo Borghi \email g.borghi@hw.ac.uk \\
      \addr Maxwell Institute for Mathematical Sciences and School of Mathematical and Computer Sciences\\
      Heriot--Watt University, Edinburgh, UK
      \AND
      \name Hyesung Im \email hi3001@hw.ac.uk \\
      \addr Maxwell Institute for Mathematical Sciences and School of Mathematical and Computer Sciences \\
      Heriot--Watt University, Edinburgh, UK
      \AND
      \name Lorenzo Pareschi \email l.pareschi@hw.ac.uk\\
      \addr Maxwell Institute for Mathematical Sciences and School of Mathematical and Computer Sciences\\
      Heriot--Watt University, Edinburgh, UK\\
      Department of Mathematics and Computer Science \\
      University of Ferrara, Italy}

\editor{Editors}

\maketitle
\begin{abstract}%
Population-based learning paradigms, including evolutionary strategies, Population-Based Training (PBT), and recent model-merging methods, combine fast within-model optimisation with slower population-level adaptation. Despite their empirical success, a general mathematical description of the collective dynamics remains incomplete. We propose a theoretical framework for neural network training based on two-time-scale multi-agent dynamics, where parameters follow fast noisy SGD/Langevin updates and hyperparameters evolve via slower selection–mutation dynamics.
 We prove the large-population limit for the joint distribution of parameters and hyperparameters and, under strong time-scale separation, derive a selection--mutation equation for the hyperparameter density. For each fixed hyperparameter, the fast parameter dynamics relaxes to a Boltzmann--Gibbs measure, inducing an effective fitness for the slow evolution. The averaged dynamics connects population-based learning with bilevel optimisation and classical replicator--mutator models, yields conditions under which the population mean moves toward the fittest hyperparameter, and clarifies the role of noise and diversity in balancing optimisation and exploration. Numerical experiments illustrate both the large-population regime and the reduced two-time-scale dynamics, and indicate that access to the effective fitness, either in closed form or through population-level estimation, can improve population-level updates. 
\end{abstract}


\tableofcontents

\section{Introduction}

Modern large-scale machine learning increasingly resembles the evolution of model populations rather than the training of a single neural network. In reinforcement learning, hyperparameter optimisation, neural architecture search, and model evolution pipelines, one repeatedly encounters systems in which multiple agents (neural networks, trajectories, or strategies) coexist, interact, and adapt over time. Examples include evolutionary strategies for black-box optimisation~\citep{salimans2017Evolution,cui2018evolutionary}, population-based reinforcement learning as in AlphaStar~\citep{vinyals2019alphastar}, and recent population-level perspectives on collective learning dynamics~\citep{bouteiller2026sociodynamics}. Further examples arise in distributed AutoML frameworks and in more recent model-merging and evolution methods that improve networks by recombining existing ones~\citep{akiba2025evolutionary,wortsman2022model}. 
These developments suggest that modern AI training pipelines cannot be understood solely in terms of single-trajectory optimisation, but require a population-level perspective based on interaction rules and evolutionary mechanisms.

Within this broader context, hyperparameter optimisation (HPO) has long been recognised as a central component of efficient machine learning pipelines. Classical approaches such as grid search and random search~\citep{bergstra2012Random} explore fixed parameter configurations, while Bayesian optimisation methods~\citep{bergstra2011Algorithms,hutter2011Sequential} introduce surrogate models to balance exploration and exploitation. Multi-fidelity strategies, including successive halving and Hyperband~\citep{li2018Hyperband}, further reduce computational cost by adaptively pruning poor configurations. While effective in many settings, these methods typically operate on independent trials and do not explicitly model interaction or adaptation across agents.

Population-based methods form a distinct paradigm within HPO and meta-optimisation. Among them, Population-Based Training (PBT), originally introduced by DeepMind~\citep{jaderberg2017Population, li2019generalized}, maintains a population of neural networks that evolve through alternating phases of gradient-based learning and fitness-driven selection. Underperforming agents are periodically replaced by copies of better ones, often with perturbed hyperparameters, enabling the joint optimisation of model parameters and hyperparameters within a single training loop. Variants of this mechanism appear in evolutionary reinforcement learning, quality-diversity optimisation, and large-scale meta-optimisation, and are closely related to evolutionary strategies and population-driven metaheuristics~\citep{dalibard2021Faster,salimans2017Evolution,lee2025Evolving,zito2025Metaheuristics,hutter2019Automated}. Comprehensive surveys of HPO trends and methodologies can be found in~\citep{archambeau2024Hyperparameter}.

Beyond explicitly evolutionary algorithms, similar population-level dynamics arise implicitly in modern large-scale learning systems. In particular, the success of ensemble-based training and selection strategies in systems such as AlphaFold~\citep{jumper2021highly} highlights how effective learning can emerge from the interaction and aggregation of multiple trained models, even when the underlying algorithm is not formulated explicitly in evolutionary terms. More recently, classical PBT replacement mechanisms have been extended to act directly on model weights: instead of replace-and-mutate updates, models may be merged through learned recombination rules~\citep{akiba2025evolutionary} or by simple averaging of fine-tuned checkpoints~\citep{wortsman2022model}. 
Taken together, these examples reveal a gap in current theory, namely the lack of a mathematical framework able to explain and analyse the collective dynamics generated by interacting populations of models.

From an optimisation perspective, population-based training can also be interpreted as a particular instance of bilevel optimisation~\citep{franceschi2018Bilevel,trillos2024CB2O}, where an inner problem performs short gradient-based updates on model parameters, while an outer mechanism selects and adapts hyperparameters or structural traits. This viewpoint naturally induces a separation of time scales and clarifies the algorithmic role of selection and mutation, but still leaves open the question of how to describe such coupled dynamics in a principled continuous-time setting.

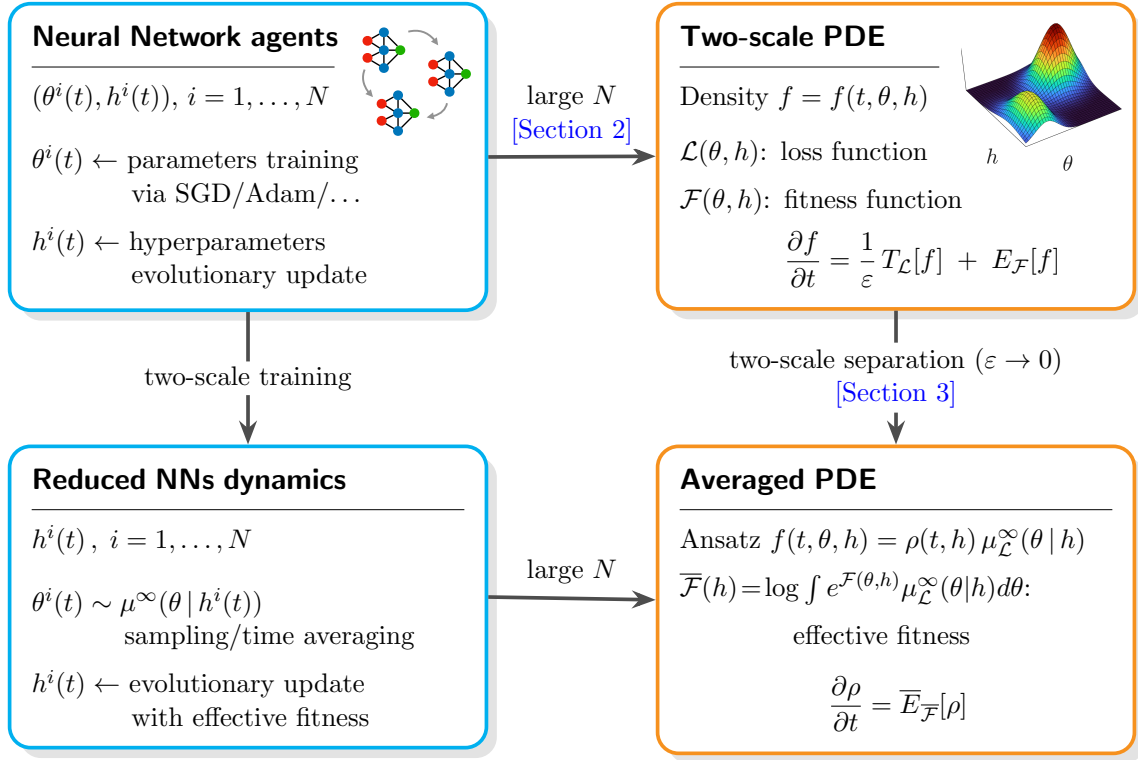
\begin{figure}[t]
  \centering
   \begin{adjustbox}{max width=\linewidth}
\begin{tikzpicture}[
  font=\normalsize,
  >=Stealth,
  label/.style={
    font=\normalsize,
    fill=white,
    fill opacity=1,
    text opacity=1,
    inner sep=2pt,
    rounded corners=2pt
  },
  basebox/.style={
    rounded corners=3mm,
    line width=1.5pt,
    fill=white,                 
    drop shadow={opacity=0.22, shadow xshift=0.9ex, shadow yshift=-0.9ex},
    inner sep=9pt,
    text width=6.2cm,
    align=left
  },
  boxBlue/.style={basebox, draw = ProcessBlue},
  boxOrange/.style={basebox, draw=BurntOrange},
  boxA/.style={boxBlue},
  boxC/.style={boxBlue},
  boxB/.style={boxOrange},
  boxD/.style={boxOrange},
  arr/.style={->, line width=1.5pt, draw=black!70}
]
\matrix (M) [
  matrix of nodes,
  row sep=19mm,
  column sep=24mm,
  nodes={anchor=north west}
]{
  |[boxA] (A)| {%
    {\large \bf\sffamily Neural Network agents}\\[-5pt]\rule{0.7\linewidth}{0.2pt}\\[2pt]
    $(\theta^i(t),h^i(t)),\, i = 1, \dots, N$\\ \bigskip
    $\theta^i(t)\gets$ parameters training  
    \\ \hspace{13mm} via SGD/Adam/\dots\\
    \medskip
    $h^i(t)\gets$ hyperparameters \\ \hspace{13mm}
    evolutionary update
  } &
  |[boxB] (B)| {%
    {\large \bf\sffamily Two-scale PDE}\\[-5pt]\rule{0.6\linewidth}{0.2pt}\\[2pt] 
    Density $f=f(t,\theta,h)$ \\ \vspace{3mm}
         $\mathcal L(\theta, h)$: loss function \\ \medskip
        $\mathcal F(\theta, h)$: fitness function \\ \medskip
    $\hspace{15mm}\displaystyle \frac{\partial f}{\partial t}=\frac{1}{\varepsilon}\,T_{\mathcal L}[f]\;+\;E_{\mathcal F}[f]$ \\ 
    \bigskip
  } \\
  |[boxC] (C)| {%
    {\large \bf\sffamily Reduced NNs dynamics}\\[-5pt]\rule{\linewidth}{0.2pt}\\[2pt]
    $h^i(t)\,,\; i = 1,\dots,N$\\ \bigskip
    $\theta^i(t) \sim \mu^{\infty}(\theta\,|\,h^i(t))$\\
    \hspace{12mm} sampling/time averaging\\ \medskip
    $h^i(t)\gets$ evolutionary update 
    \\ \hspace{13mm} with effective fitness  \\ \medskip
  } &
  |[boxD] (D)| {%
    {\large \bf\sffamily Averaged PDE}\\[-5pt]\rule{\linewidth}{0.2pt}\\[2pt]
    Ansatz $f(t,\theta,h)=\rho(t,h)\,\mu_\LL^{\infty}(\theta\,|\,h)$ \medskip \\
  $\overline{\mathcal F}(h)\! = \!\log\int e^{\mathcal F(\theta, h)} \mu_\LL^{\infty}(\theta|h) d\theta$:\,\,  \\ \medskip
  \hspace{15mm}  effective fitness
    \\ \bigskip
         {\hspace{20mm} $ \displaystyle \frac{\partial  \rho}{\partial t} =\overline{E}_{\overline{\mathcal F}}[\rho]$ }  \\ 

  } \\
};
%
\draw[arr] (A.east) -- (B.west)
  node[midway, label, yshift=17pt,align=center] {large $N$\\
  {\color{blue}[Section \ref{sec:2}]}};
\draw[arr] (C.east) -- (D.west)
  node[midway, label, yshift=11pt] {large $N$};
\draw[arr] (A.south) -- (C.north)
  node[midway, label, xshift=0pt] {two-scale training};
\draw[arr] (B.south) -- (D.north)
  node[midway, label, xshift=0pt, align=center] {two-scale separation ($\varepsilon \to 0$)\\
  {\color{blue}[Section \ref{sec:twoscale}]}};
\begin{scope}[overlay]
  \node[inner sep=0, anchor=north east]
    at ($(A.north east)+(-1mm,-2mm)$)
    {\includegraphics[width=19mm]{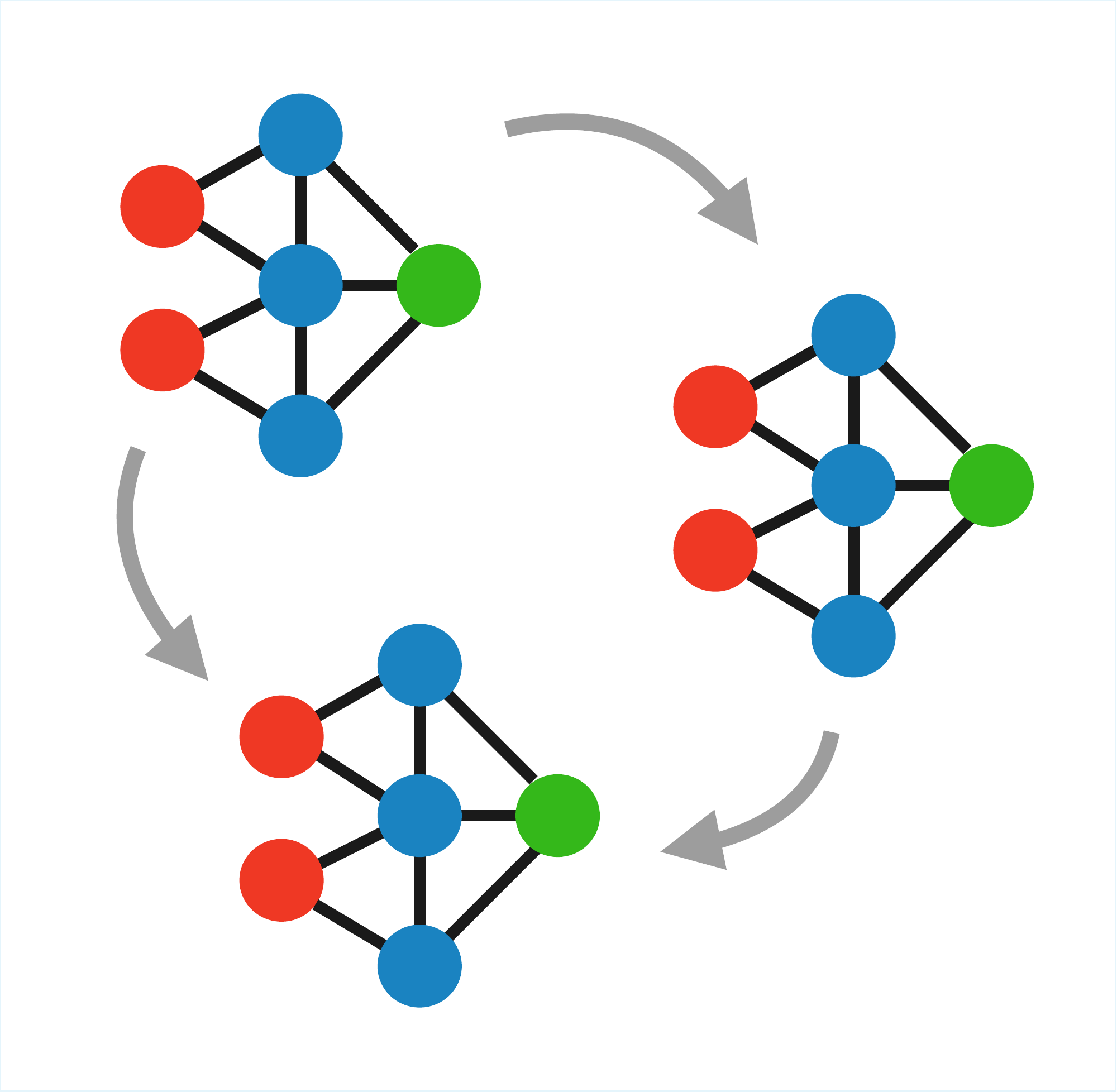}};
      \node[inner sep=0, anchor=north east]
    at ($(B.north east)+(-2mm,-3mm)$)
    {\includegraphics[width=23mm]{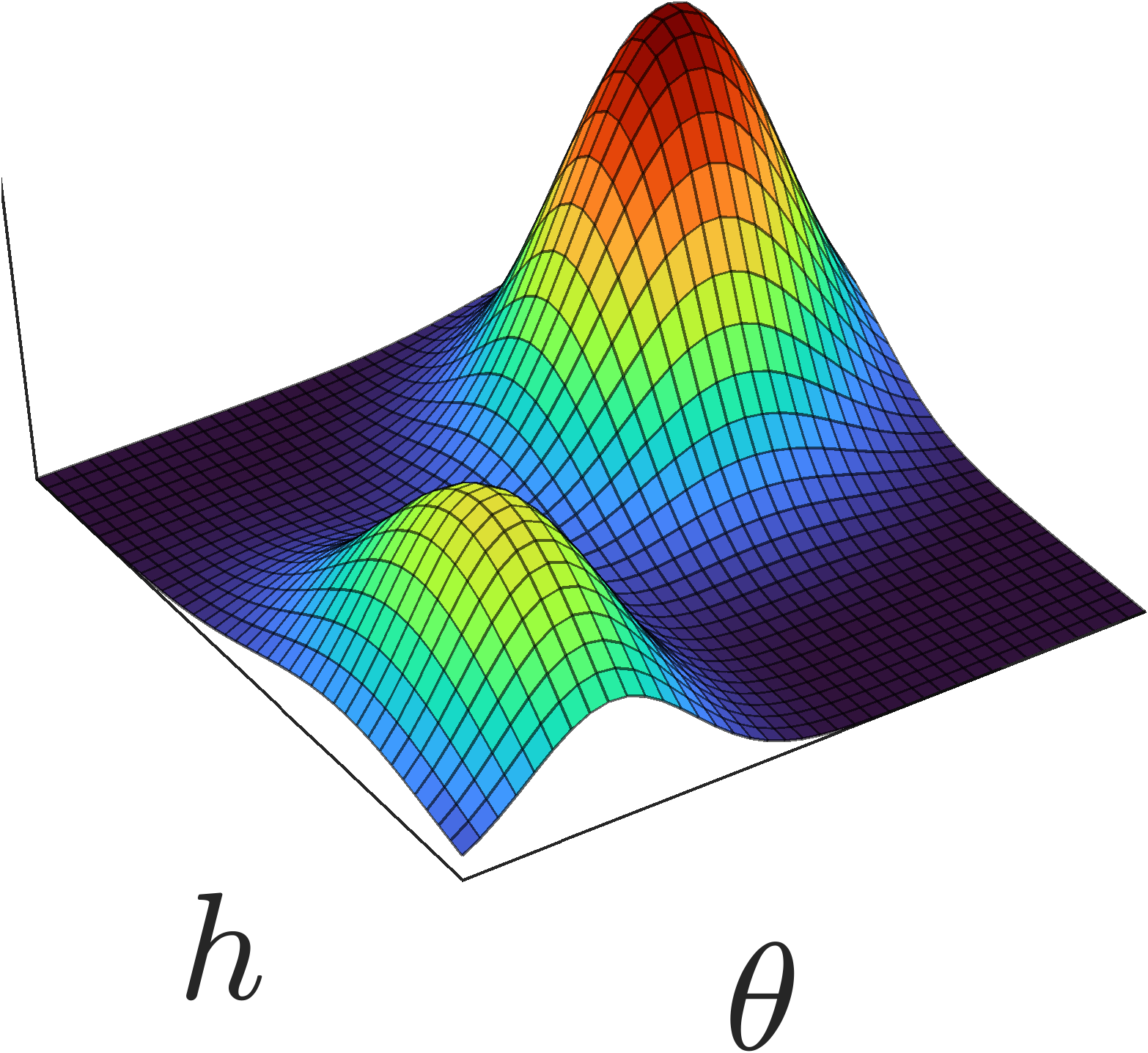}};
\end{scope}
\end{tikzpicture}%
\end{adjustbox}
  \caption{Two-scale modelling pipeline. Large-population limits ($N\to\infty$) connect the agent-based neural-network dynamics to a two-scale kinetic PDE for the joint density 
$f(t,\theta,h)$ over parameters and hyperparameters, in which 
a training operator and a selection--mutation 
evolutionary operator act on different time scales. 
A two-scale separation ($\varepsilon\to 0$) averages out the fast parameter dynamics, which relaxes to a Boltzmann--Gibbs equilibrium $\mu^\infty_{\mathcal{L}}(\theta|h)$, and yields a closed averaged PDE for the hyperparameter density $\rho(t,h)$ driven by the effective fitness $\overline{\mathcal{F}}(h)$.}
\end{figure}

Despite its conceptual appeal and practical impact, PBT still lacks a rigorous continuous mathematical formulation. The original method is algorithmic, heuristic, and strongly dependent on implementation details such as mutation schedules, population size, and replacement policies. Yet its underlying principles, selection, mutation, adaptation, and interaction, are deeply connected to well-established mathematical theories, including interacting particle systems, kinetic equations, and mean-field limits. Recent works have already begun to explore this direction for related swarm-based optimisation methods~\citep{duong2020mean,borghi2025kinetic,borghi2024unified,fornasier2021cbo,carrillo2024fedcbo,benfenati2022binary,carrillo2021cbo,LuTadmorZenginoglu2024Swarm}, showing that kinetic and mean-field modelling provides a powerful language for capturing high-dimensional optimisation and evolutionary search. This is different from the mean-field modelling considered in~\citep{mei2018mean,sirignano2020mean}, where the training of a single neural network is studied in the limit of infinitely many parameters.

The goal of this manuscript is to provide a rigorous continuous-time theory for neural populations undergoing joint optimisation and evolution. We introduce a continuous-time framework for large-population neural training in which agents undergo gradient learning, selection, and mutation.

From a theoretical point of view, minibatch stochastic gradient descent admits a well-known interpretation as a discretised Langevin diffusion~\citep{mandt2017sgd,li2019stochastic,Chizat2022}, whose stationary distribution is asymptotically Gibbsian~\citep{chaudhari2018sgd}. This observation provides a natural foundation for two-time-scale modelling approaches, in which fast stochastic gradient dynamics rapidly equilibrate model parameters, while slower population-level mechanisms act on hyperparameters or higher-level traits. At a macroscopic level, such dynamics are closely related to classical selection--mutation models and population transport equations~\citep{ackleh2016population,alfaro2019evolutionary,bench2022weak,delmoral2001Asymptotic,delmoral2001stability,perthame2007transport}. This separation enables a mathematically tractable description, clarifies the role of diversity and stability, and connects evolutionary heuristics to replicator--mutation dynamics and transport equations.

Our main contributions are:
\begin{itemize}
    \item \textbf{Large-population limit (Theorem~\ref{t:chaos}).} 
    We introduce a continuous-time 
interacting agent model for PBT in which 
network parameters evolve via Langevin/SGD dynamics and hyperparameters 
evolve via selection--mutation. We rigorously prove that the 
 multi-agent system converges to a two-scale kinetic PDE 
description as the population size $N \to \infty$ via coupling techniques and a truncated Wasserstein distance.

    \item \textbf{Two-time-scale reduction (Section~\ref{sec:twoscale}).} Under strong 
    time-scale separation ($\varepsilon \to 0$), we derive an averaged selection--mutation 
    PDE for the hyperparameter density $\rho(t,h)$, driven by an \emph{effective fitness} 
    $\overline{\mathcal{F}}(h)$ that averages the fitness against the Boltzmann--Gibbs 
    equilibrium of the fast parameter dynamics. We show this reduction connects 
    PBT to bilevel optimisation (Lemma~\ref{l:penalization}) and to 
    classical replicator--mutator equations from mathematical biology 
    (Section~\ref{sec:perthame}).

    \item \textbf{Convergence guarantee (Theorem~\ref{t:convergence}).} We prove that, 
    under a uniqueness assumption on the effective fitness maximiser, the population mean 
    converges toward the optimal hyperparameter configuration at an exponential rate, 
    provided the selection pressure is sufficiently large. The result makes the roles of 
    noise strength, population diversity, and selection pressure explicit and quantitative.

    \item \textbf{Algorithms and experiments (Section~\ref{sec:numerics}).} We propose a 
    reduced training algorithm (Algorithm~\ref{alg:macro}) in which the inner SGD loop is 
    replaced by sampling from, or time-averaging over, the parameter equilibrium. 
    Experiments on quadratic and Himmelblau bilevel problems, and on a CartPole deep 
    reinforcement learning task, validate the large-population limit, the two-time-scale 
    regime, and the practical benefit of effective fitness estimation.
\end{itemize}


The remainder of the paper is organised as follows. In Section \ref{sec:2}, we introduce the stochastic particle model describing a population of neural networks undergoing both training and evolutionary dynamics. We also derive the corresponding two-scale PDE model in the many-agent limit $N \to \infty$, first formally and then rigorously. In Section \ref{sec:twoscale}, we analyse the two-time-scale limit, leading to a reduced macroscopic selection--mutation dynamics for hyperparameters and its connection with bilevel optimisation and replicator--mutator models. In Section \ref{sec:4}, we study this reduced dynamics and identify assumptions under which the population mean moves toward the fittest hyperparameter configuration. Numerical examples are presented in Section \ref{sec:numerics}, and concluding remarks and perspectives are discussed in Section \ref{sec:outlook}.

\section{Modelling population-based training}
\label{sec:2}

\subsection{The collective training mechanism}

Population-Based Training (PBT) aims to find the best Neural Network (NN) for a given task by optimising both the set of parameters $\theta$ and of hyperparameters $h$ which we assume take values in $\Theta \subseteq \RR^{d_\theta}, H \subseteq \RR^{d_h}$ respectively. The number of parameters is typically of several order of magnitude larger than the number of hyperparameters, $d_\theta \gg d_h$.

PBT employs a set of $N\in \mathbb{N}$ NNs, which we will also call \textit{agents}
\[
(\theta^i, h^i)\in \Theta \times H  \qquad i = 1,\dots,N
\]
that are iteratively updated according to two mechanisms:
\begin{itemize}
\item Loss $\LL$ minimisation via independent learning dynamics;
\item Fitness $\F$ maximisation via collective interaction.
\end{itemize}

The loss minimisation problem is solved by gradient-based algorithms such as SGD, RMSprop or ADAM-type optimisers. While the learning dynamics is used only to update the parameters $\theta^i$,  it typically depends on the hyper-parameters $h^i$ since they could include, for instance, the learning step, regularisation or noise strength. To stress this dependence we consider a loss $\LL = \LL(\theta, h)$ which depends on both $\theta$ and $h$.

In the case of SGD training, the parameter dynamics can be approximated by a stochastic differential equation (SDE) of the form
\begin{equation} \label{eq:sgd}
d\theta^i = -\nabla_{\theta} \LL(\theta^i, h^i)\,dt + \sqrt{2\beta^{-1}(h^i)}\Sigma^{1/2}(\theta^i)\,dB^i_t,
\end{equation}
where $(B_t^i)_{t\geq 0}$ are independent $d_\theta$-dimensional Brownian motions. The matrix $\Sigma(\theta)$ is positive definite and captures the variance induced by stochastic loss sampling,
while $\beta(h)>0$ determines the noise strength, and typically depends on the hyperparameters via the learning rate \citep{cheng2020stochastic,li2019stochastic,li2021validity}. RMSprop and Adam adaptive training algorithms can be also modelled via SDE of type, by adding auxiliary random variables into the dynamics \citep{malladi2022sdes}, see Remark \ref{rmk:adam} for more details.

The fitness function $\F$ is conceptually different from the loss $\LL$, as it is not directly related to the training dynamics. For instance, while $\LL$ may depend on the training data, $\F$ typically evaluates the generalisation performance of an agent-NN on a validation dataset. Moreover, the fitness function does not need to be differentiable. In general, it depends on both variables: $\F = \F(\theta,h)$.

The fitness maximisation is addressed by PBT via genetic-type dynamics where the best performing agents (that is, the \textit{fittest}) are copied with a certain rate. To keep the number of agents limited, the worst performing ones are discarded, or they are discarded randomly. 
At a slower rate (with respect to the inner learning dynamics) each agent $(\theta, h)$ is subject to a genetic update of type
\begin{equation} \label{eq:coll}
\begin{split} 
\theta' &= \theta_*  \\
h' &= h_*  + \sigma \xi 
\end{split}
\end{equation}
where $(h_*,\theta_*)$ is an agent sampled among the best ones in the following way
\[
\mathbb{P}\left((\theta_*,h_*) = (\theta^i,h^i)\right) = \frac{\exp(\F(\theta^i,h^i))}{\sum_{j=1}^N \exp(\F(\theta^j,h^j))}\,.
\]
Note that the parameter $\theta$ is copied to $\theta_*$, and mutation acts only at the hyper-parameter level via a mutation vector $\xi \sim \mathcal{N}(0, I_d)$ with strength $\sigma>0$. If the search space $H$ is bounded and convex, an additional projection is added to ensure $h^i \in H$ at all times, by modifying $h'$ in \eqref{eq:coll}  as $h' \gets \textup{Project}_H(h_* +\sigma \xi)$.

The resulting algorithm can be seen as a continuum training of the parameters $\theta^i$, interrupted by genetic updates (or \textit{jumps}) where $(\theta^i, h^i)$ are resampled among the fittest agents. See Algorithm \ref{alg:pbt} for a high-level description of the PBT evolution.

In Algorithm \ref{alg:pbt} the parameter $\tau \in (0,1)$ determines the frequency of such evolutionary interruptions.  For $\tau = 1$ we have that all agents perform a genetic update simultaneously. For $\tau \ll 1$, interruptions are more frequent, but agents are less likely to undergo resampling. Note that for any $\tau \in (0,1]$ the expected total number of jumps per agent per unit time is always 1.

\begin{algorithm}[t]
\caption{Population-Based Training}
\label{alg:pbt}
\begin{algorithmic}
\State Initialize $N\in \mathbb{N}$ NNs with parameters $\theta^i\in \Theta$ and hyperparameters $h^i\in H$, $i = 1,\dots, N$
\State Set $\tau\in (0,1]$,$\sigma>0$, $t_{\max}>0$
\State $t \gets 0$
\While{$t < t_{\max}$}
\For{$i = 1,\dots, N$} 
\State Train $\theta^i$ via SGD \eqref{eq:sgd} for $[t, t + \tau)$
\Comment{parallel NN training}
\State $F^i \gets \F(\theta^i,h^i)$ \Comment{evaluate fitness}
\EndFor
    \For{$i = 1, \dots, N$}                      \Comment{population interaction} 
  \State with probability $\tau$:       
    \State \qquad select a NN $(\theta^j,h^j)$ proportionally to $\exp(F^j)$   
    \State \qquad sample random mutation vector $\xi$
     \State \qquad $\tilde \theta^i \gets \theta^j$  \Comment{copy parameter}
    \State \qquad $\tilde h^i \gets h^j + \sigma\xi$   \Comment{copy and perturb hyperparameter}
        \State \qquad $\tilde h^i \gets \textup{Project}_{H}(\tilde h^i)$   \Comment{project to search space}
    
 \EndFor
  \State $\theta^i\gets \tilde \theta^i, h^i \gets \tilde h^i$ for all $i=1,\dots,N$
  \State $t \gets t + \tau$
\EndWhile
\end{algorithmic}
\end{algorithm}

\subsection{Description in the many-agent limit} \label{sec:manyagent}

When dealing with a large interacting system of $N \gg 1$ agents $(\theta^i,h^i), i = 1,\dots,N$, it is convenient to provide a statistical description of the system. Rather than tracking the microscopic trajectories of each agent $(\theta^i, h^i)$, we aim to derive a model which quantifies the probability of finding a certain agent with features $(\theta,h)$ at a given time $t\geq 0$.
The model can be formally obtained by taking the many agent limit $N \to \infty$. 

We consider a probability distribution  $f = f(t, \theta, h)$ in the parameter/hyperparameter phase space $\Theta \times H \subseteq \RR^{d_\theta}\times \RR^{d_{h}}$ which we assume it well approximates the empirical distribution 
\begin{equation} \label{eq:chaos}
f (t, \theta, h) \;\approx \; f^N(t, \theta, h) = \frac1N \sum_{i=1}^N \delta(\theta - \theta^i(t))\delta(h - h^i(t))\qquad \textup{for}\quad N \gg 1\,,
\end{equation}
where $\delta(\cdot)$ is the Dirac delta distribution centred at 0.

When a random variable $\theta^i$ evolves according to a SDE of type \eqref{eq:sgd}, the evolution of its probability law is determined by a Fokker--Planck operator.
In \eqref{eq:sgd}, we note that only the parameters $\theta^i$ evolve, while the hyperparameters $h^i$ are fixed. Therefore, for $f = f(t, \theta, h)$, the Fokker--Planck operator associated with the training dynamics is given by
\begin{equation} \label{eq:opT}
T_{\LL}[f]
:= \nabla_\theta \cdot \bigl(\nabla_\theta \LL(\theta,h)\, f \bigr)
+ \beta^{-1}(h)  \nabla_\theta^2 : (\Sigma (\theta) \,f ) \,,
\end{equation}
see, for instance, \citep[Chapter 6]{bogachev2022fokker}. Above, $A:B = \mathrm{tr}(A^\top B)$ is the usual Frobenius inner product. For Adam-based training, a similar Fokker--Planck operator can be also obtained via the corresponding SDE formulation, see Remark \ref{rmk:adam}.


It is during the collective genetic-type jump that agents interact. Assumption \eqref{eq:chaos} can also be understood as a so-called \textit{propagation of chaos} assumption on the agents \citep{sznitman1991topics}. This means that for large population sizes the correlation between the agents become weaker and weaker, until they become independent as $N \to \infty$. As agents are indistinguishable, this also implies that they are identically distributed with law $f$. 

If a genetic update occurs at time $t$, the probability density of copying an agent
with state $(\theta,h)$ is then given by
\[
g(t,\theta,h)
=
\frac{e^{\F(\theta,h)}\,f(t,\theta,h)}
{\iint e^{\F(\theta',h')}\,f(t,\theta',h')\,d\theta'\,dh'} \,.
\]
Instead, the hyperparameter mutation in \eqref{eq:coll} corresponds to a convolution
with a probability kernel $K$ in the $h$ variable, namely
\[
(K_\sigma *_h g)(t,\theta,h) = \int g(t,\theta, h+\sigma\xi)\,K(d\xi)\,.
\]

The selection--mutation mechanism outlined in Algorithm~\ref{alg:pbt} therefore
corresponds, at the distribution level, to the reweighting--convolution operator
\begin{equation} \label{eq:opE}
E_\F[f] := K_\sigma *_h \left( \frac{e^{\F} f} {\int e^{\F} df } \right) - f\, .
\end{equation}
The term $-f$ is required because a new agent is created due to the jump, while the
pre-jump agent is deleted in the process. This ensures that the total amount of agents does not change during the computation.


As discussed earlier, the parameter $\tau \in (0,1]$ describes how frequently each agent has the possibility of undergoing a jump. It can be interpreted as a time-discretisation parameter for the collective evolutionary dynamics where, in the limit $\tau \to 0$, jumps may occur at any time and asynchronously.
As $N \to \infty$ and $\tau \to 0$, we formally obtain a continuum description of Algorithm~\ref{alg:pbt}, in which the time evolution of the distribution $f$ is determined by both the training and evolutionary operators
\begin{equation} \label{eq:fpde}
\frac{\partial f}{\partial t} = T_{\LL}[f] + E_\F[f] \,.
\end{equation}
Equations of this type are inspired by recent applications of kinetic theory \citep{PareschiToscani2013} and are 
closer in spirit to kinetic formulations of simulated annealing \citep{Pareschi2024, Chizat2022} and genetic algorithms \citep{borghi2025kinetic}.

\begin{remark} \label{rmk:adam}
The training operator for adaptive algorithms such as RMSprop or Adam can be equivalently derived from their continuous SDE approximations, as for instance in \citep{malladi2022sdes}. The main difference with respect to SGD is that the parameter state space must be augmented to account for the momentum variable $m$, or for other auxiliary variables appearing in the dynamics.

For instance, in the case of Adam the SDE reads
\[
\begin{cases}
    d\theta^i = - P^{-1}(h^i,u^i) m^i\,dt, \\
    d m^i = \alpha(h^i)\bigl(\nabla_{\theta} \LL(\theta^i, h^i) - m^i\bigr)\,dt + \sqrt{2\beta^{-1}(h^i)}\,\Sigma^{1/2}(\theta^i)\,dB^i, \\
    d u^i = \gamma(h^i)\bigl(\operatorname{diag}(\Sigma(\theta^i)) - u^i\bigr)\,dt,
\end{cases}
\]
where $\alpha(h)$ and $\gamma(h)$ are hyperparameters and $P(h,u)$ is a preconditioning matrix. At the PDE level, the associated distribution must also be extended and takes the form $f = f(t,\theta,h,m,u)$. As will be discussed in Section~\ref{sec:twoscale}, these auxiliary variables become irrelevant in the two-time-scale limit, simplifying the description.
\end{remark}

\begin{remark}
A related line of work studies the optimal design of training protocols by casting reduced learning dynamics as an optimal control problem, often in simplified teacher–student models where the dynamics closes on a few order parameters (yielding low-dimensional ODEs \citep{Mignacco2025}). Our approach is complementary in scope and object: rather than controlling a single reduced trajectory, we derive an averaged population dynamics for the hyperparameter distribution $\rho(t,h)$, obtained from the large population limit and a two-time-scale (fast training/slow selection–mutation) limit. This averaged population equation can be regarded as a state equation for future protocol-design problems, including but not limited to optimal control.
\end{remark}

\subsection{Propagation of chaos and convergence}

We show now that the connection between Algorithm \ref{alg:pbt} and the two-scale PDE model \eqref{eq:fpde} can be made rigorous. We will consider weak measure solutions to \eqref{eq:fpde}, see Definition \ref{def:sol} for a precise characterisation. We restrict to analysis to the case of SGD parameters training with isotropic noise.

\begin{assumption} \label{asm:chaos} $\,$
    \begin{enumerate}
    \item Parameters are trained with SGD via \eqref{eq:sgd} with $\beta>0$ fixed and $\Sigma(\theta) \equiv I$.
    \item The loss $\LL \in C^1(\RR^{d_\theta} \times \RR^{d_h})$ is globally Lipschitz continuous. 
    \item  The fitness $\F\in C(\RR^{d_\theta} \times \RR^{d_h})$ is globally Lipschitz continuous and bounded with $0< \inf{\F}<\sup\F<\infty$\,.
    \end{enumerate}
\end{assumption}

To compare the empirical distribution $f^{N,\tau}$ (we stress now also the dependence on $\tau\in(0,1)$) and the distribution $f$, we will employ the Bounded-Lipschitz (BL) norm between measures, also known as flat, or Dudley metric \citep{dudley2018real}. For any Borel measure $\mu\in \mathcal{M}(\RR^d)$ it is defined as 
\begin{equation}\label{eq:BL}
    \|\mu\|_{{\mathrm{BL}}} :=\sup \left \{ \int \phi\, d \mu \;  \middle |\; \| \phi\|_\infty < 1/2, \|\phi\|_{\textup{Lip}}<1   \right\} 
\end{equation}
where $\|\phi\|_\infty  := \sup_x|\phi(x)|$ and $\|\phi\|_\textup{Lip} := \sup_{x\neq y}|\phi(x) - \phi(y)|/|x-y|$.

\begin{theorem} \label{t:chaos}
Let Assumption \ref{asm:chaos} hold, $t_{\max}>0$ be a computational time horizon, and $f_0\in \mathcal{P}_q(\RR^{d_\theta}\times \RR^{d_h}), q>1$ an initial distribution. Consider the multi-agent system generated by Algorithm \ref{alg:pbt} $\{(\theta^i(t),h^i(t))_{t\in [0, t_{\max}]}\}_{i=1}^N$  
with $f^{N,\tau}(t)$ being the associated empirical measure and $(\theta^i(0), h^i(0))$ being independent and $f_0$-distributed.

Then, there exists a weak measure solution  $f\in C([0, t_{\max}], \mathcal{P}_q(\RR^{d_\theta}\times \RR^{d_h} ))$ to \eqref{eq:fpde} with initial data $f_0$, and it holds
\[
\mathbb{E} \| f^{N, \tau}(t) -  f(t)\|_{\mathrm{BL}} \longrightarrow 0 \qquad \textup{as} \quad  N \to \infty, \;\tau \to 0
\]
for all $t \in [0, t_{\max}]$.

\end{theorem}

Since $f(t)$ is deterministic, Theorem \ref{t:chaos} implies that the stochastic multi-agent system becomes deterministic in the limit $N\to \infty$. Moreover, the result validates the ansatz $f^N\approx f$ \eqref{eq:chaos} used in the derivation of the kinetic PDE \eqref{eq:fpde}. We prove the theorem by first considering the limit $N \to \infty$, via a coupling technique and the optimal transport formulation of the BL norm. Then, the limit $\tau \to 0$ is considered by exploiting the stability both the training and evolutionary operators with respect to the BL norm. We refer to Appendix \ref{app:chaos} for the details.

\subsection{Modelling extensions}
\label{sec:extensions}

In genetic algorithms, many variants of selection procedures have been proposed, see \citep{khalid2013selection} for a review. In more sophisticated mechanisms, poorly performing agents are more likely to undergo a genetic update, or an agent’s quality is determined by its rank within the population rather than by its fitness value.

\subsubsection{Substitution of NNs performing poorly}
In Algorithm \ref{alg:pbt}, every NN is equally likely to be updated, with probability $\tau$ at every time $t_n = n \tau$. The genetic update frequency could also be made fitness-dependent by prescribing a higher replacement probability for agents that have poor fitness values as done in \citep{jaderberg2017Population}. This amounts, at the PDE level, to modifying the loss term in \eqref{eq:opE}. In this case, agents are no longer selected uniformly for removal, but rather with a bias favouring the least fit ones. A natural choice is to use the same exponential weighting as in the gain term, but applied to the negative fitness $-\F(\theta,h)$. This more elaborate algorithmic strategy would lead to the operator
\begin{equation*}
\widehat E_\F[f]
:= K_\sigma *_h \left( \frac{e^{\F} f}{\int e^{\F} df} \right)
- \frac{e^{-\F} f}{\int e^{-\F} df}\,.
\end{equation*}

\subsubsection{Rank-based selection}
In Algorithm  \ref{alg:pbt}, the probability of a NN agent $(\theta^i, h^i)$ to be selected for duplication is proportional to its fitness $ \F(\theta^i, h^i)$. Another common strategy consists of picking agents with respect to their rank among the ensemble. For instance, one can decide that only the top $20\%$ NN are copied when performing a genetic update, see \citep{jaderberg2017Population}. 

Mathematically, given a group of NNs $(\bm{\theta}, \bm{h}) = \{(\theta^i, h^i)\}_{i=1}^N$, the function rank is defined as
\[
\mathrm{Rank}(\theta^i, h^i\,|\, \bm{\theta}, \bm{h}) := \#\{ j \,|\, \F(\theta^j, h^j) \geq\F(\theta^i, h^i)  \}\,.
\]
This can be extended for an arbitrary probability distribution $f\in \mathcal{P}(\RR^{d_\theta}, \RR^{d_h})$ as
\[
\mathrm{Rank}(\theta, h \,|\, f) :=  \int_{\F(\tilde \theta, \tilde h)\geq \F( \theta,  h)} f( d\tilde h, d \tilde h)\,.
\]
Therefore, at least formally, it is possible to derive a similar two-scale PDE description of the collective training dynamics where rank-based selection is used, as done in Section \ref{sec:manyagent} for fitness-based selection. A rigorous convergence analysis is more delicate, and we leave it for future work. In the literature, rank-based particle interactions have been considered in \citep{bench2022weak,demircigil2025convergence}.

\section{Two-time-scale limit} 
\label{sec:twoscale}

The two processes, loss minimisation and fitness maximisation, happen at two different time scales. Typically, several training steps are performed (of order $10^4$ \citep{jaderberg2017Population}, for instance) before a single evolutionary update. This also reflects the dimensionality unbalance between the number of parameters and of hyperparameters.
To balance the two processes we introduce a scaling parameter $\ve>0$ and study the rescaled model 
\begin{equation} \label{eq:fpde:eps}
\frac{\partial f}{\partial t} = \frac1{\ve} T_{\LL}[f] + E_\F[f] \,.
\end{equation}
At the algorithmic level, this corresponds to rescaling the time of the parameter optimisation \eqref{eq:sgd}, making the loss training longer of a factor $1/\ve$.

We are interested in studying the evolution of the hyperparameter distribution
\[
\rho(t,h) = \int f(t,\theta,h) d \theta
\]
by averaging out the fast time evolution for $\ve\ll 1$.
In the limit $\ve \to 0$, equation \eqref{eq:fpde:eps} reduces to 
\[
T_{\LL}[f] = 0\,.
\]
which implies that $f$ should be at equilibrium with respect to $T_{\LL}$, taking the form 
\begin{equation}\label{eq:ansatz}
f(t,h,\theta) = \rho(t,h)\,\mu^\infty_\LL(\theta|h)\,.
\end{equation}
with $T_{\LL}[\mu^\infty(\cdot|h)] = 0$. This means that, given a certain hyperparameter $h\in H$, the distribution on the space of parameters is fixed in time, and given by the conditional probability $\mu^\infty_\LL(\theta|h)$. At this slow time scale, only the hyperparameters distribution $\rho = \rho(t,h)$ evolves.

The specific form of the equilibria $\mu^\infty_\LL$ depends on the training dynamics used. Since $T_{\LL}$ corresponds to the Fokker--Planck operator, they typically take the form of exponential distributions associated with a given potential. For instance, when considering SGD with isotropic noise ($\Sigma \equiv I_d$ in \eqref{eq:sgd}), the equilibrium is given by a Gibbs distribution associated with the loss function:
 \begin{equation} \label{eq:gibbs}
\mu^{\infty}_\LL(\theta|h) = \frac{e^{-\beta(h)\LL(\theta,h)}}{\int e^{-\beta(h)\LL(\theta,h)}d \theta}\,,
 \end{equation}
see, for instance, \citep[Proposition 4.6]{pavliotis2014chapter4}. Existence of an invariant distribution and convergence towards it is related for Fokker--Planck operator in more general settings have been investigated in \citep{ottobre2011asymptotic,achleitner2015FP,mandt2017sgd}.  Also, they might not have an explicit form due to the possible noise degeneracy.  Specific results for SGD, RMSProp, and Adam-type algorithms have been obtained in \citep{mandt2017sgd,suzuki2023adam}.

Under ansatz \eqref{eq:ansatz}, if we test the solution $f$ against $\phi=\phi(h)$ we obtain
$\langle \phi, T_{\LL}[f]\rangle = 0$ and
\begin{align*}
\langle \phi, E_\F[f]\rangle 
& = \iiint \left( \frac{e^{\F(\theta,h)}}{\int e^\F\, df}\, \phi(h + \sigma \xi ) - \phi( h) \right) K(\xi)\,
\rho(t,h)\,\mu^{\infty}_\LL(\theta| h)\,d \xi\, d \theta\, d h \\
& = \iint \left(\int \frac{e^{\F(\theta,h)}}{\int e^\F\, df}\,\mu^{\infty}_{\LL}(\theta|h)\,d \theta \right)\phi(h + \sigma \xi )\, K(\xi)\,\rho(t,h)\,d \xi\, d h
\;-\;
\int \phi(h)\,\rho(t,h)\,dh .
\end{align*}

Let $\overline{\F} = \overline{\F}(h)$ be what we call the effective Gibbs-averaged fitness
\begin{equation} \label{eq:Fbar}
    \overline{\F}(h):= \log \int e^{\F(\theta,h)}\mu^{\infty}_\LL(\theta|h) d \theta\,.
\end{equation}
Note that the normalizing constant remains unchanged
\begin{align*}
    \int e^\F\, df 
&= \iint e^{\F(\theta, h)}\,\rho(t,h)\,\mu^{\infty}_\LL(\theta|h)\,d\theta\,dh \\
&= \int \exp\!\left(\log\int e^{\F(\theta,h)}\mu^{\infty}_\LL(\theta|h)\,d \theta \right)\rho(t,h)\,dh 
= \int e^{\overline{\F}}\,d \rho\,,
\end{align*}
leading to 
\begin{equation}\label{eq:Ebar}
\langle \phi, E_\F[f]\rangle 
=
\langle \phi, \overline{E}_{\overline{\F}}[\rho]\rangle
:=
\iint \left(
\frac{e^{\overline{\F}(h)}}{\int e^{\overline{\F}} \, d \rho}\,\phi(h + \sigma \xi )
- \phi(h)
\right) K(\xi)\,\rho(t, h)\,d \xi\, d h .
\end{equation}

Altogether, we obtain in the limit $\ve \to 0$ the averaged equation for $\rho = \rho(t,h)$
\begin{equation} \label{eq:rhopde}
\frac{\partial \rho}{\partial t} = \overline{E}_{\overline{\F}}[\rho]\,.
\end{equation}
Therefore, in the two-time scale limit the dynamics reduces to a pure selection-mutation evolution for the effective fitness  $\overline{\F}$. Equations of this type have been studied in \citep{delmoral2000branching,delmoral2004book} to study the many agent limit of genetic-type algorithms via Feynman--Kac distribution flows. 

\begin{remark} \label{rmk:effectiveestensions}
The key point in the derivation of the averaged PDE \eqref{eq:rhopde} is the equilibrium ansatz \eqref{eq:ansatz}. The emergence of equilibrium in parameter space depends only on the scale separation between training and evolution, and its precise form depends on the specific algorithm used for parameter training. Therefore, the type of evolutionary strategy employed does not affect the emergence of the lower-dimensional PDE. The same reduction can thus be carried out, for instance, for rank-based selection or for the other extensions discussed in Section \ref{sec:extensions}.
\end{remark}

\begin{remark} 
For stochastic optimisation algorithms, the long-time behaviour may exhibit non-equilibrium features, including persistent probability currents \citep{chaudhari2018sgd,li2022what}. In \citep{li2022what}, the authors relate this behaviour to the implicit regularisation properties of SGD by showing the existence of slow dynamics along manifolds of minimizers. In our modelling framework, this corresponds to having an implicit dependence in the Fokker--Planck operator $T_{\LL}$ by $\ve$. In this case, more accurate macroscopic models can de derived by performing a Chapman--Enskog-type expansion \citep{bobylev2020chapter7} of the training operator.
\end{remark}

\subsection{Connections to bi-level optimisation problems}

We shed some light on the role of the effective fitness  $\overline{\F}$ by relating it to bi-level optimisation problems. The NN training, indeed, can be cast as a problem of type
\begin{equation} \label{eq:bilevel}
\begin{split}
  h^\star &\in  \underset{h \in H} {\arg\!\max}\, \F(\theta^\star(h),h )\\
 &\textup{subject to}\quad \theta^\star(h) \in \underset{\theta \in \Theta}{\textup{argmin}} \,\LL( \theta,h)
  \end{split}
\end{equation}
where one sequentially aims to minimize the loss $\LL$ and maximize the fitness $\F$ by tuning parameters and hyperparameters. For each fixed $h$, $\overline{\F}(h)$ fits  into the variational–inference viewpoint: the equilibrium measure $\mu^{\infty}_\LL$ acts as a prior, possibly concentrating on minimisers of $\LL(\cdot,h)$,
while $\F(\theta,h)$ plays the role of a log-likelihood.

Now, let us consider the case where the equilibrium is explicitly given by the Gibbs measure \eqref{eq:gibbs}, with fixed $\beta>0$. We stress the dependence on $\beta$ by writing the averaged loss as $\overline{\F}_\beta$.
Let $\mathcal{H}$ be the log-entropy, we rewrite the KL divergence as
\begin{align*}
\textup{KL}(\nu \,|\,\mu_{\LL}^{\infty}) & = \int \nu(\theta) \log(\nu(\theta))d \theta - \int\nu(\theta) \log\left(\frac{e^{-\beta\LL(\theta,h)}}{\int e^{-\beta \LL}d\theta} \right) d\theta \\
& = \mathcal{H}[\nu] + \beta \mathbb{E}_\nu[\LL(\theta,h)] + \log\int e^{-\beta \LL}d\theta
\end{align*}
Let us denote 
$ \overline{\LL}_{\beta}(h):=  -(1/\beta) \log  \int e^{-\beta \LL(\theta,h)}d \theta$ 
for which, by the Laplace principle \citep{dembo2010}, it holds
\begin{equation*} \label{eq:Lbeta}
\overline{\LL}_{\beta}(h)\longrightarrow \inf_\theta \LL(\theta,h)\qquad \textup{as} \quad \beta \to \infty\,.
\end{equation*}

We obtain a further characterisation of the effective fitness  as
\begin{equation*} \label{eq:Fbeta}
\overline{\F}_\beta(h)= \sup_{\nu}\Big(\mathbb{E}_\nu\left[\F(\theta,h) - \beta \left(  \LL(\theta,h) - \overline{\LL}_{\beta}(h)\right) \right]  - \mathcal{H}[\nu]\Big)\,.
\end{equation*}
which has a clear interpretation in terms of the bi-level optimisation problem \eqref{eq:bilevel}. 
Indeed, in view of \eqref{eq:Lbeta}, we can see $\overline{\F}_\beta(h)$ as an entropy regularized version of a penalized fitness 
\[
\P_\beta(h) := \F(\theta,h) - \beta \left(\LL(\theta,h) - \inf_\theta \LL(\theta,h) \right)
\]
where $\beta$ here plays the role of the penalisation parameter. This is a standard approach in solving the bi-level optimisation problem, where \eqref{eq:bilevel} is viewed as a constrained optimisation problem, see, for instance, \citep{wright2023fully}, and for which it holds $\max_h \P_\beta(h)\to \max_h\F(h,\theta^\star(h))$ as $\beta \to \infty$.
The following lemma quantifies this intuition and show that the same holds for the averaged problem $\max_h \overline{\F}_\beta(h)$. See Appendix \ref{app:addproofs} for the proof.

\begin{assumption}\label{asm:laplace}
Assume the hyperparameter space $H\subset\RR^{d_h}$ to be compact and the loss function $\LL = \LL(\theta, h)$ to satisfy:
\begin{enumerate}
\item For every $h\in H$, $\LL(\cdot, h)\in C^4(\RR^{d_\theta})$,
    \item For every $h\in H$, the map $\theta\mapsto \LL(\theta,h)$ admits a unique global minimizer $\theta^\star(h)\in\RR^{d_\theta}$ at which the Hessian is positive definite: there exists $\lambda>0$ such that 
    \[\nabla_\theta^2\LL(\theta^\star(h),h)\succeq \lambda I_{d_{\theta}} \qquad \textup{for all} \quad h \in H\,.\]
   
    \item For every $r>0$,
    \[
    \inf\Bigl\{\LL(\theta,h)-\LL(\theta^\star(h),h): h\in H,\ |\theta-\theta^\star(h)|\ge r\Bigr\}>0.
    \]
    
    \item The Gibbs weights have uniformly integrable first moments:
    \[
    \sup_{h\in H}\int_{\RR^{d_\theta}}|\theta| \,
    \exp\bigl(-(\LL(\theta,h)-\LL(\theta^\star(h),h))\bigr)\,d\theta<\infty.
    \]
\end{enumerate}
\end{assumption}

 \begin{lemma} \label{l:penalization}
 Let $\F, \LL$ satisfy Assumption \ref{asm:chaos} and \ref{asm:laplace}. Then, for all $\beta\ge 1$ the Gibbs-averaged fitness \eqref{eq:Fbeta} satisfies
 \[
 |\overline{\F}_\beta(h) - \F( \theta^\star(h),h)| \lesssim   \frac{1}{\sqrt{\beta}} \qquad \textup{for all}\quad  h \in H.
 \]
\end{lemma}

The proof is based on the quantitative Laplace principle derived in \citep{hasenpflug2024wasserstein}, where the authors provide a rate for the limit $\mu^{\infty}_\LL(\cdot|h) \to \delta_{\theta^\star(h)}$, as $\beta \to \infty$, in terms on Wasserstein distances. Recall $\beta>0$ is the inverse of the temperature in the Gibbs measures $\mu^{\infty}_\LL$. For simplicity, we applied the result in the  case where $\LL(\cdot, h)$ attains a unique global minimum for each $h\in H$, but the assumptions can be made more general to include manifolds of solutions, see \citep{hasenpflug2024wasserstein} for more details.
Assumption \ref{asm:laplace} (4) ensures that the Gibbs equilibria are well-defined, while (3) ensures that there is no diverging sequence of local minima $\theta^k$ such that $\LL(\theta^k, h)\to \min\LL(\cdot, h)$ as $k \to \infty$.

\subsection{Diffusion–reaction approximation in the slow mutation regime}
\label{sec:perthame}

To relate \eqref{eq:rhopde} with the evolutionary PDEs typically used for biological systems  \citep{perthame2007transport,ackleh2016population,alfaro2019evolutionary}, we consider the following re-scaling for $\neweps>0$
\begin{equation} \label{eq:scaling}
t \to t/\neweps\,, \quad \sigma \to \sqrt{\neweps} \sigma\,, \quad \overline{\F} \to \neweps \overline{\F}\,.
\end{equation}

Under this rescaling of the fitness $\overline{\F}$, we obtain as $\neweps \to 0$
\[
e^{\neweps \overline{\F}} = 1 + \neweps \overline{\F} + \mathcal{O}(\neweps^2)
\]
and, by integrating, also
$
\int e^{\neweps \overline{\F}}d \rho =  1 + \neweps \int \overline{\F} d \rho + o(\neweps)\,.
$
Using the formula:
$
(1 + \neweps x)/(1 + \neweps y)  = 1 + \neweps(x - y) + o(\neweps)
$
we get that the exponential weights can be approximated as
\[
\frac{e^{\neweps \overline{\F}}}{\int e^{\neweps \overline{\F}}d\rho} = 1 + \neweps \left( \overline{\F} -   \langle \overline{\F}, \rho  \rangle \right) +o(\neweps)\,.
\]
The Taylor expansion for a smooth test function $\phi = \phi(h)$ instead gets us
\begin{align*}
\int \phi(h') K(d \xi)& = \int \phi(h + \sqrt{\neweps} \sigma\xi, \theta ) K(d \xi) \\
& = \int \left(\phi(h) + \sqrt{\neweps} \sigma \xi \cdot \nabla \phi(h) + \frac12 \neweps \sigma^2 \xi^\top \nabla^2 \phi(h)\xi \right)  K(d \xi) \\
&  = \phi(h) + \neweps \frac{\sigma^2}2 \Delta \phi(h) + o(\neweps)\,.
\end{align*}
where we used that $\mathbb{E}[\xi] = 0$, $\mathbb{E}[\xi \tilde{\xi}^\top ] = I$ for $\xi, \tilde\xi\sim K$.
Altogether we obtain 
\begin{align*}
 \frac{e^{\neweps \overline{\F}(h)}}{Z_{\neweps \overline{\F}}[\rho]} \int \phi(h+\sigma\xi)K(d \xi) 
 & = ( 1 + \neweps(\overline{\F}(h) - \langle \overline{\F}, \rho \rangle)\left(\phi(h) + \neweps \frac{\sigma^2}{2} \Delta \phi \right) + o(\neweps) \\
 & = \phi(h) - \neweps \left( \overline{\F}(h) - \langle \overline{\F}, \rho\rangle \right) \phi(h) + \neweps \frac{\sigma^2}2 \Delta \phi (h) + o(\neweps)\,,
\end{align*}
which leads, in turn, to
\[
\langle \phi, \overline{E}^\neweps_{\overline{\F}}[\rho] \rangle    = \left \langle K_{\sqrt{\neweps}\sigma} \ast \phi,\frac{e^{\neweps\overline{\F}}}{Z_{\neweps\overline{\F}}[\rho]} \rho \right \rangle  - \langle \phi, \rho\rangle 
 =   \neweps \left(\left\langle \phi, \overline{\F} \rho\rangle  - \langle \overline{\F}, \rho \rangle \langle \phi,  \rho \right \rangle\right)  + \neweps \frac{\sigma^2}2 \langle \Delta \phi, \rho \rangle  + o(\neweps)\,.
\]

Therefore, \eqref{eq:rhopde} under the slow selection rescaling \eqref{eq:scaling} formally leads in the limit $\neweps \to 0$ to a diffusion-reaction equation
\begin{equation} \label{eq:perthame}
\frac{\partial  \rho}{\partial t}  =   
\rho \left(\overline{\F}(h) - \langle\overline{\F},\rho \rangle \right)
+ \frac{\sigma^2}2\Delta\rho \,,
\end{equation}
 which has been largely studied in the context of evolutionary equations in mathematical biology \citep{perthame2007transport}.
The mutation mechanism is now described by a Laplacian term, while selection is described by a source term where particles at $h$ are created with rate $\overline{\F}(h) $ destroyed uniformly with a uniform rate $\langle\overline{\F}, \rho \rangle$. Note that the terms are balanced so that the total mass remains constant during the evolution.  Equations of type \eqref{eq:perthame} are known in the sampling literature also as birth-death or replicator-mutation dynamics \citep{lu2023birth}, and can be cast as gradient flows with respect to the Hellinger--Kantorovich (Wasserstein--Fisher--Rao) geometry when additional drift terms are present \citep{liero2018optimal,gallouet2019unbalanced}.
The slow mutation limit therefore reveals a remarkable connection between genetic-type algorithms and gradient flows.

\section{Mean convergence toward the fittest solution}
\label{sec:4}

In this section we study the evolution of the averaged PDE \eqref{eq:rhopde} to get an understanding on the hyperparameters evolution at large times. 
Existence of measure solutions to \eqref{eq:rhopde} have been studied in \citep{delmoral2000branching, delmoral2004book,borghi2026chaos}.
In \citep{delmoral2002stability} the authors show that averaged evolutionary operator \eqref{eq:Ebar} is stable, and solutions to \eqref{eq:rhopde} exponentially converge to the unique steady state, under suitable assumptions. This is achieved by using semigroup techniques and Dobrushin’s ergodic coefficient, but an explicit characterisation of the steady states is missing. 

Here, we are interested here in understanding whether we can expect the hyperparameters to converge towards the fittest solution with respect to the averaged objective
\begin{align} \label{eq:singleopt}
h^\star \in \underset{h\in \RR^{d_h}}{\mathrm{argmax}}\; \overline{\F}(h)\,.
\end{align}
To do so, we propose an analysis based on bounds on the moments evolutions, and a quantitative version of the Laplace principle. 

\subsection{Moments evolution}

Denote with $\mup(\rho)$ the mean of $\rho\in \mathcal{P}(\RR^{d_h})$,
$\mup(\rho) : = \int h \,\rho(dh)\,$, and the particle energy, or second moment, as 
$
\mathrm{En}(\rho): = (1/2)\int|h|^2\rho(dh)\,.
$
For notational convenience we also define the action of reweighing the agents distribution $\rho$ with respect to the average fitness $\overline{\F}$ as
\begin{equation}\label{eq:G}
G_{\overline{\F}}[\rho](dh) :=  \frac{e^{\overline{\F}(h)} \rho(dh)}{\int e^{\overline{\F}(h')}\rho(dh')} \,.
\end{equation}
We note that the evolutionary operator \eqref{eq:Ebar} can be compactly be written as $\overline{E}_{\overline{\F}}[\rho] = K_\sigma\ast G_{\overline{\F}}[\rho]$.

Along solutions $\rho = \rho(t, h)$ of the averaged PDE \eqref{eq:rhopde}, it holds
\begin{equation}\label{eq:meanevo}
\frac{d}{dt}\mup(\rho) = \mup\left(G_{\overline{\F}}[\rho] \right)  - \mup(\rho)\,.
\end{equation}
This can be shown by directly applying the weak formulation \eqref{eq:Ebar} of the evolutionary operator. Since the mutation vectors are mean-zero, mutation does not affect the mean evolution, which is fully determined by the selection mechanism. 

The evolution of the energy $\mathrm{En}(\rho)$ is instead given by
\begin{align*}
\frac{d}{dt} \mathrm{En}(\rho(t)) &  = \iint |h + \sigma\xi|^2  G_{\overline{\F}}[\rho(t)](dh)  - \mathrm{En}(\rho(t)) \\
& = \mathrm{En}\left( G_{\overline{\F}} [\rho(t)] \right)   - \mathrm{En}(\rho(t)) + d_h \sigma^2 \,,
\end{align*}
where we used that $\int |\xi|^2 K(d \xi ) = d_h$ since $K$ is the standard $d_h$-dimensional Gaussian measure.
From the moments' evolution, then, we understand that any equilibrium $\rho^\infty$ should satisfy the equations
\begin{equation} \label{eq:equilibrium}
\begin{split}
\mup(\rho^\infty) &= \mup\left(G_{\overline{\F}}[\rho^\infty]\right) \\
\mathrm{En}(\rho^\infty) &= \mathrm{En}\left(G_{\overline{\F}}[\rho^\infty]\right) +  d_h \sigma^2 \,.
\end{split}
\end{equation}
It is also interesting to note that there is a dimensionality dependence on the energy equation coming from the diffusion term.

\subsection{Mean convergence}

As it is clear from \eqref{eq:equilibrium}, the properties of the selection operator $G_{\overline{\F}}[\cdot]$ are crucial for the long time behaviour of the averaged PDE model \eqref{eq:rhopde}. At the algorithmic level, we can expect that if $\overline{\F}$ attains sharp maxima, few agents will be considered fit for being copied and mutated, leading to fast concentration. On the contrary, if $\overline{\F}$ is almost flat, all the agents will considered almost equally fit, leading to a more dissipative dynamics. 

In Genetic Algorithms, this property of the evolution is called \textit{selection pressure} \citep{cerf1998ga,haasdijk2018pressure}.
In our settings, since the weights in $G_{\overline{\F}}[\cdot]$  are proportional to $\exp(\overline{\F})$, the selection pressure can be tuned by introducing an additional parameter $\alpha$ to rescale the fitness as
\[
\overline{\F} \gets \alpha\overline{\F}\qquad \textup{with}\quad \alpha >0\,. 
\]
Large $\alpha$'s correspond to high selection pressure, while lower ones to a low selection pressure. 

In this section, we show that under suitable assumptions on $\overline{\F}$ (in particular, it attains a unique maximizer) and if the selection pressure $\alpha$ is sufficiently high, solutions to the averaged PDE models \eqref{eq:rhopde} converge towards the fittest point $h^\star$ of problem \eqref{eq:singleopt}.
The analysis technique mirrors the one proposed in \citep{borghi2025kinetic} for genetic algorithms. The only differences are that here we are considering a time-continuous dynamics, there is no crossover mechanism, and we are maximizing instead of minimizing.

\begin{assumption} \label{asm:fitness} The objective function $\overline{\F}:\RR^{d_h}\to \RR$ is continuous and satisfies:
\begin{enumerate}[label=(\roman*)]
\item \textit{(solution uniqueness)} there exists a unique global maximizer $h^\star$;
\item \textit{(growth conditions)} there exists $L_{\overline{\F}},c_u, c_l, R_l>0$ such that 
\begin{equation*}
\begin{cases}
    |\overline{\F}(h) - \overline{\F}(\tilde h)|\leq L_{\overline{\F}}( 1+ |h| + |\tilde h|)|h - \tilde h| & \forall\; h, \tilde h\in \RR^{d_h} \\
    \overline{\F}(h^\star) - \overline{\F}(h) \leq c_u (1 + |h|^2) & \forall\; h\in \RR^{d_h} \\
     \overline{\F}(h^\star) - \overline{\F}(h) \geq c_u |h|^2 & \forall \; h \,:\, |h|>R_l\,; 
\end{cases}
\end{equation*}
\item \textit{(inverse continuity)} there exists $c_p, p>0$, $R_p >0$ and an upper bound $\overline{\F}_\infty>0$ such that 
\begin{equation*}
\begin{cases}
    c_p|h - h^\star| \leq \overline{\F}(h^\star) - \overline{\F}(h)    & \forall\;h\,:\, |h|\leq R_p \\
      \overline{\F}_\infty \leq \overline{\F}(h^\star) - \overline{\F}(h)    & \forall\;h\,:\, |h|> R_p \,.
\end{cases}
\end{equation*}
\end{enumerate}
\end{assumption}

The first important estimate refer to the effect of the selection operator on the energy of a given distribution $\rho\in \mathcal{P}(\RR^{d_h})$. Under Assumption \ref{asm:fitness} (ii) the following estimate holds \citep[Lemma 3.3]{carrillo2021cbo}:
\begin{equation} \label{eq:energy}
\mathrm{En}\left(G_{\alpha \overline{\F}}[\rho] \right)
\leq b_1 + b_2 \mathrm{En}(\rho)\,,
\end{equation}
with 
\[
b_1 = R_l + b_2\,, \qquad b_2 = 2\frac{c_u}{c_l}\left( 1 + \frac{1}{\alpha c_l}\frac1{R_l^2} \right) \,.
\]
Note that the estimate does not  deteriorates
as the selection pressure $\alpha$ increases.
The next important result we will employ is a quantitative version of the Laplace principle derived in \citep[Proposition 21]{fornasier2024consensus}. Assume $h^\star \in \textup{supp}(\rho)$, and let $\overline{\F}_r:= \overline{\F}(h^\star) - \inf_{h\in B(h^\star, r)}\overline{\F}(h)$, for any $r\in (0,R_p]$ and $q>0$ such that $q +\overline{\F}_r < \overline{\F}_\infty$, we have
\begin{equation}\label{eq:laplace}
| \mup\left(G_{\alpha\overline{\F}}[\rho] \right) - h^\star| \leq c_p (q + \overline{\F}_r)^{1/p}  + \frac{\exp(-\alpha q)}{\rho\left(B(h^\star, r)\right)}\int|h - h^\star|\rho(dh)\,.
\end{equation}
The above estimate relates the mean of the reweighted probability $G_{\alpha \overline{\F}}[\rho]$ with the solution to the optimisation problem \eqref{eq:singleopt}. Under Assumption \ref{asm:fitness}, it is possible to take the $r>0$ very small such that the first term is smaller with respect to a given tolerance. Then, by taking a sufficiently large selection pressure $\alpha \gg 1$, the second term can also be made arbitrary small. The proof of the following convergence result is based on this intuition, see Appendix \ref{app:laplace} for the details.

\begin{theorem}
    \label{t:convergence}
Consider a weak measure solution $\rho \in C([0,t_{\max}],\mathcal{P}_2(\RR^{d_h}))$ to \eqref{eq:rhopde}, and let the effective fitness  $\overline{\F}$ satisfy Assumption \ref{asm:fitness}. If $h^\star\in \mathrm{supp}(\rho_0)$, then, for any tolerance $tol>0$, there exists a selection pressure $\alpha = \alpha(\overline{\F},\rho_0,t_{\max})>0$ sufficiently large such that 
\[
|\mup(\rho(t)) - h^\star|^2 \leq e^{- t}|\mup[\rho_0] - h^\star|^2 + (1 - e^{-t}) \frac{tol}2 \qquad \textup{for all} \quad t\in [0,t_{\max}]\,.
\]
Moreover,  if 
\[t_{\max} = t^\star \geq \log \left(\frac{2|\mup(\rho_0) - h^\star|^2}{tol} \right)  \]
it holds
\[ |\mup(\rho(t^\star)) - h^\star|^2 \leq  tol\,.\]
\end{theorem}

\begin{remark} 
In the present work we focus on the convergence scenario associated with a single optimum of the effective fitness,
which, intuitively, leads to a unimodal equilibrium distribution around the fittest solution. In more general settings, though, one can expect
a richer long-time behaviour, with potentially non-unique, multimodal steady states, and phase-transition-type phenomena as selection pressure varies. The relation discussed in Section \ref{sec:perthame}  with descent dynamics with respect to the Hellinger--Kantorovich (Wasserstein--Fisher--Rao) geometry suggests the availability of Lyapunov functionals, and a variational characterisation of steady states. A systematic analysis of these entropy-based mechanisms and possible phase transitions is beyond the scope of the present paper and will be addressed in future work.    
\end{remark}


\section{Numerical experiments}
\label{sec:numerics}

In this section, we numerically investigate the many-agent limit and the emergence of two time scales in the training-evolutionary dynamics. We begin with low-dimensional examples and then validate the proposed framework on a Reinforcement Learning (RL) task. For the illustrative examples, we consider a quadratic bi-level optimisation problem, for which the equilibria can be computed exactly, allowing us to simulate both the limiting dynamics as $N \to \infty$ and the two-time-scale regime $(\ve \to 0)$. We then consider the Himmelblau function, which exhibits a more complex landscape, to study the interplay between parameter training and hyperparameter evolution under Population-Based Training (PBT) updates. Finally, through the deep RL example of the CartPole problem, we illustrate the effect of the two-time-scale structure via effective fitness evaluations and the evolution of the hyperparameter distribution over time. Since sampling parameters from the equilibrium distribution of the parameters is not possible in this setting, we average the fitness over a time window. Assume we need to evaluate the effective fitness at a step $k$ of a neural network $(\theta_k^i, h^i_k)$, this can be approximated as
\begin{equation} \label{eq:timefitness}
 \overline{\F}(h_k) \approx \frac1m \sum_{s = 1}^m \F(\theta^i_{k-s}, h^i_{k})\,,\qquad \textup{for a time window} \quad m\geq 1\,,
\end{equation}
where $\theta^i_{k-s}, s = 1,\dots, m$ are the last iterations of the training algorithm (Adam optimizer in our case). Note that, for relatively small time windows $m$, we have $h^i_{k} = h^i_{k-s}$ since the hyperparameters are updated at a slower rate. Our experiments suggests that larger averaging windows yield more stable and accurate results, supporting the interpretation that the hyperparameter PBT update does not directly optimize the instantaneous fitness, but rather an effective fitness induced by the parameter training equilibrium.

\subsection{Illustrative examples}

We consider two bi-level optimisation problems in which the fitness function depends only on the parameters, namely $\F=\F(\theta)$ with $\theta\in\RR^2$. The hyperparameter space is also two-dimensional, with $h=(h_0,h_1)\in\RR^2$. The first hyperparameter, $h_0$, will be used to introduce a bias in the loss function $\LL = \LL(\theta, h_0)$ used for training. 
The second hyperparameter, $h_1$, determines the noise strength in the SGD training dynamics
\begin{equation}\label{eq:toy-langevin}
d\theta_t=-\nabla_\theta \mathcal{L}(\theta,h_0)\,dt+h_1\,dB_t\,.
\end{equation}
Under this choice, the optimal hyperparameters are given by $h_0=0$, while $h_1$ should decrease over the course of the computation. With these settings, the agents \(\{(\theta^i,h^i)\}_{i=1}^N\) evolve according to Algorithm~\ref{alg:pbt} with $\tau = 1$ and mutation noise $\sigma = 0.1$. Both parameters and  hyperparameters are initially sampled uniformly over $[-1,1]$.
In the PBT update of the hyperparameters, the agents to be copied are selected proportionally to $\exp(\alpha \F(\theta^i))$, with selection pressure $\alpha = 100$. This strategy is also known as Gibbs or softmax-based selection. The parameters evolution \eqref{eq:toy-langevin} is discretized via explicit Euler--Maruyama scheme with time step $\Delta t = 0.01$.

\subsubsection{Quadratic function}

We consider a quadratic fitness, similar to the one in \citep{jaderberg2017Population}
\begin{equation*}
\mathcal{F}(\theta)=1.2-(\theta_0^2+\theta_1^2)\,, \quad \theta=(\theta_0, \theta_1)\in \RR^2,
\label{eq:quadF}
\end{equation*}
which attains its maximum $1.2$ at $\theta=(0,0)$. The loss function corresponds to a bias fitness 
\begin{equation*}
\LL(\theta,h_0):= -\F(\theta -h_0 ) = -1.2+\bigl((\theta_0-h_0)^2+(\theta_1-h_0)^2\bigr)\,.
\end{equation*}
Note that $\LL(\cdot,h_0)$ is also quadratic, and so \eqref{eq:toy-langevin} admits a unique Gaussian invariant measure,
\begin{equation}\label{eq:toy-steady}
\mu^\infty_{\LL}(\theta\mid h) = \frac{e^{- 2\LL(\theta,h_0)/{h_1^2}}}{\int e^{-{2}\LL(\theta,h_0)/{h_1^2}}d\theta }
= \mathcal N\!\left(
\begin{pmatrix} h_0 \\ h_0 \end{pmatrix},
\frac{h_1^2}{4} I_2
\right).
\end{equation}


\begin{figure}[t]
    \centering
    \begin{subfigure}[t]{0.45\textwidth}
        \centering
        \includegraphics[width=\textwidth]{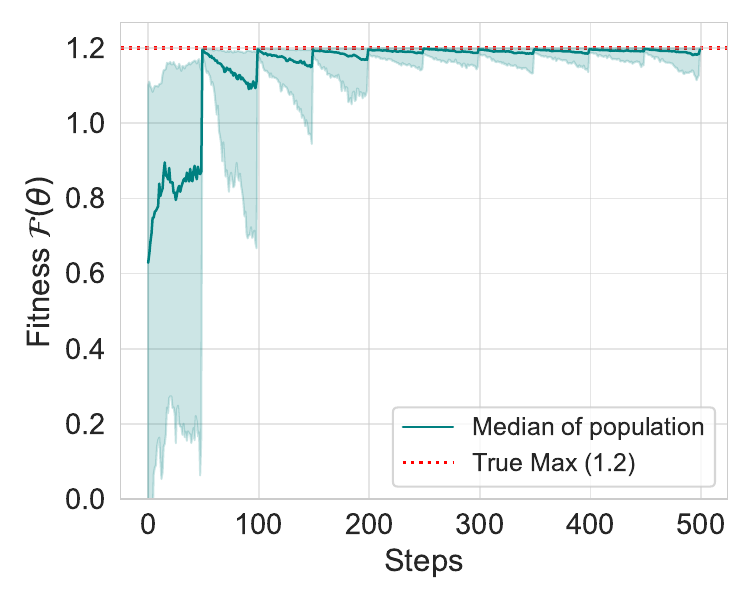}
        \caption{Fitness}
    \end{subfigure}
    \begin{subfigure}[t]{0.45\textwidth}
        \centering
        \includegraphics[width=\textwidth]{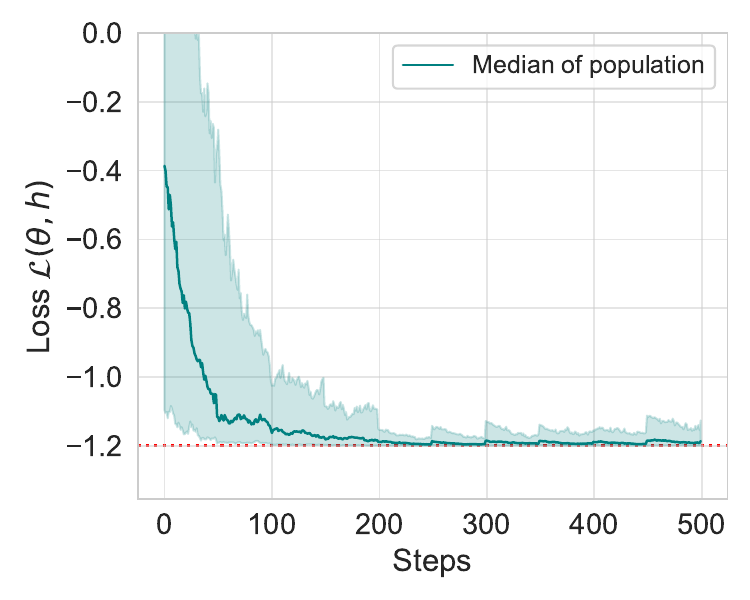}
        \caption{Loss}
    \end{subfigure}
    \caption{Quadratic objective fitness and loss.
    Median and $10\%-90\%$ quantiles of the population fitness $\F(\theta)$ and loss $\LL(\theta,h)$ over $100$ iterations of Algorithm \ref{alg:pbt}. PBT updates, where agents are resampled according to their fitness values, occur every 50 training steps and lead to a discontinuity in the fitness evolution.
 Population size is $N=100$.}
    \label{fig:objval}
\end{figure}

\begin{figure}[t]
    \centering
    \begin{subfigure}{\textwidth}
        \centering
        \includegraphics[trim = {0 4mm 0 0}, clip,width=0.9\textwidth]{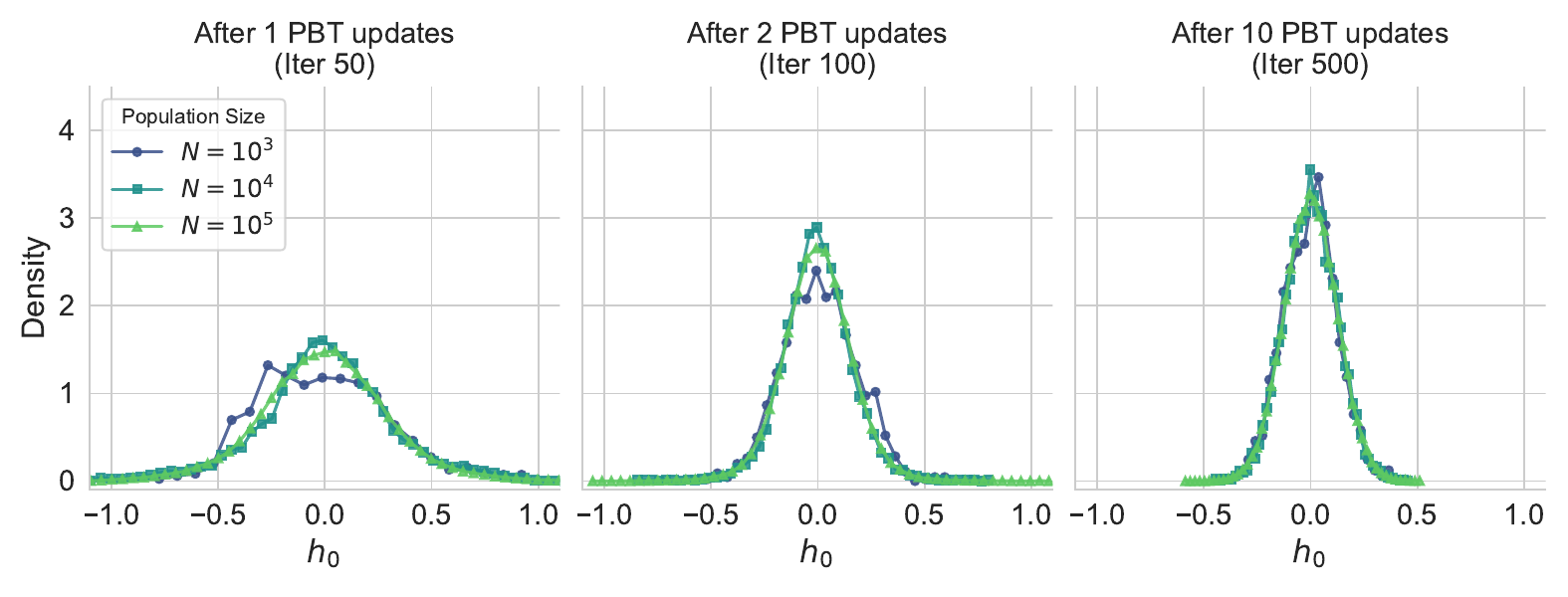}
        \caption{Evolution of $h_0$ marginal distribution}
        \label{fig:n-h0}
    \end{subfigure}
\smallskip

     \begin{subfigure}{\textwidth}
        \centering
        \includegraphics[trim = {0 4mm 0 0}, clip, width=0.9\textwidth]{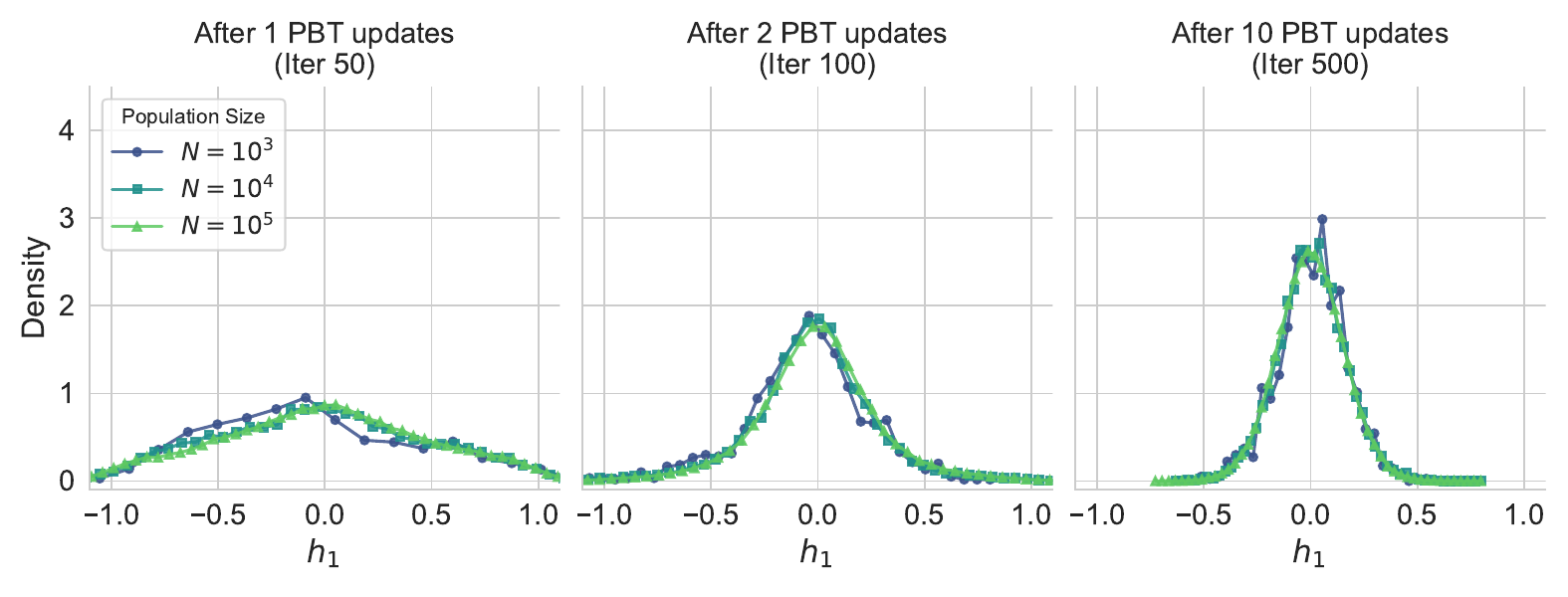}
        \caption{Evolution of $h_1$ marginal distribution}
        \label{fig:n-h1}
    \end{subfigure}
    \caption{Many-agent limit $N\to \infty$. Evolution of the hyperparameter distributions for different population sizes \(N=10^3,10^4,10^5\) and quadratic problem. Agents are initially sampled uniformly from $[-1,1]$ and evolved according with Algorithm \ref{alg:pbt}.  The total number of iterations is $500$, with selection jumps at every 50 training steps. The curves are constructed by connecting bin centres (markers) with straight lines and the number of bins is automatically determined.}
    \label{fig:N}
\end{figure}

\begin{figure}[t]
    \centering
    \begin{subfigure}{\textwidth}
        \centering
        \includegraphics[trim = {0 4mm 0 0}, clip,width=0.9\textwidth]{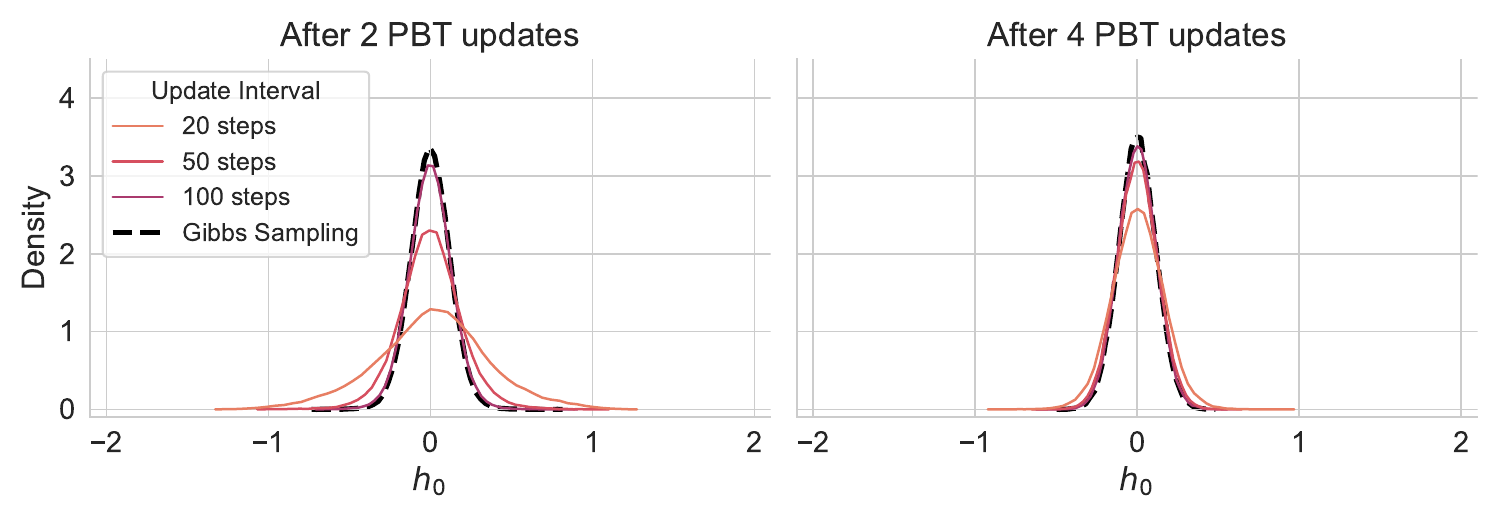}
        \caption{Evolution of $h_0$ marginal distribution}
        \label{fig:h0}
    \end{subfigure}
\smallskip

     \begin{subfigure}{\textwidth}
        \centering
        \includegraphics[trim = {0 4mm 0 0}, clip,width=0.9\textwidth]{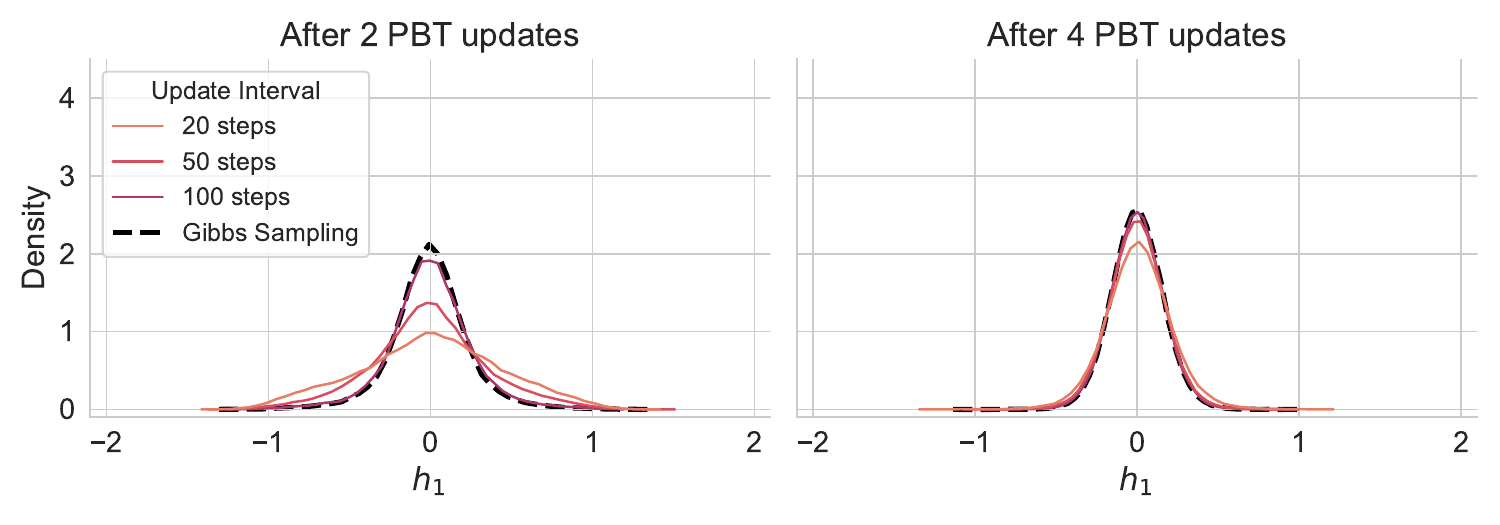}
        \caption{Evolution of $h_1$ marginal distribution}
        \label{fig:h1}
    \end{subfigure}
\caption{Fast training limit $\ve \to 0$.
Evolution of the hyperparameter distributions for different fast-training regimes. The limit $\ve \to 0$ in the PDE model \eqref{eq:fpde}  corresponds
in Algorithm~\ref{alg:pbt} to increasing the number of internal SGD steps performed between successive PBT hyperparameter updates. We consider 20, 50, and 100 inner training steps. The case $\ve = 0$ corresponds to directly sampling from the equilibrium distribution \eqref{eq:toy-steady} and evolving the agents according to Algorithm~\ref{alg:macro}, which simulates the averaged PDE model \eqref{eq:rhopde}. The population size is $N=10^5$, and the curves are obtained by connecting bin centers with straight lines using 50 bins.}
    \label{fig:two-time}
\end{figure}

Figure~\ref{fig:objval} shows how the population fitness and loss evolve during training and across PBT updates with $N = 100$ agents. Immediately after each PBT update, the median fitness exhibits a sharp jump toward the optimum, as the top-performing agents are copied. The fitness then gradually decreases until the next update. This occurs because the SGD/Langevin dynamics optimise the biased loss function. When exploration noise is added to the hyperparameters $h$, the Langevin dynamics drive the parameters $\theta$ toward these mutated, suboptimal targets, temporarily worsening the true fitness, while the biased loss continues to decrease.

Figure~\ref{fig:N} illustrates the many-agent limit of the PBT dynamics for $N = 10^3, 10^4, 10^5$. As predicted by Theorem~\ref{t:chaos}, the evolution becomes less noisy as $N$ increases, suggesting convergence toward the deterministic kinetic PDE model \eqref{eq:fpde}. Since the mutation strength does not decrease during the computation, the large-time distribution of the hyperparameters is approximately Gaussian and centered around the optimal values $h_0 = 0$ and $h_1 = 0$. This is consistent with the convergence analysis of Theorem~\ref{t:convergence}.

Figure~\ref{fig:two-time} presents the two-time-scale limit in which the SGD training dynamics is rescaled by a parameter $\ve>0$. At the algorithmic level, this corresponds to increasing the number of inner training steps performed between successive evolutionary hyperparameter updates. We consider 20, 50, and 100 training steps per PBT update, corresponding to the rescaled PDE evolution \eqref{eq:fpde} in the regime $\ve \to 0$. The limiting case $\ve = 0$, corresponding to the averaged PDE \eqref{eq:rhopde}, is simulated using Algorithm~\ref{alg:macro}, in which the iterative inner training loops are replaced by direct sampling from the invariant measure~\eqref{eq:toy-steady}.

Simulating the averaged equation significantly reduces the computational cost. Indeed, the dominant cost of PBT, the inner-loop training, is replaced by a single sampling step from a Gibbs distribution. Moreover, we observe that, as more training steps are performed between consecutive PBT updates, the hyperparameter distribution more closely matches the one produced by the reduced algorithm. Also, the two-time-scale separation seems not to be uniform in time as the difference between the distributions becomes less evident after 4 PBT updates, see Figure \ref{fig:two-time} (right).

\begin{algorithm}[t]
\caption{Training via averaged PDE \eqref{eq:rhopde}}
\label{alg:macro}
\begin{algorithmic}
\State Initialize $N\in \mathbb{N}$ NNs with hyperparameters $h^i\in H$, $i = 1,\dots, N$
\State Set $\tau\in (0,1]$,$\sigma>0$, $t_{\max}>0$
\State $t \gets 0$
\While{$t < t_{\max}$}
\For{$i = 1,\dots, N$} 
\State $\overline{F}^i \gets \F(\theta^i,h^i)$ with $\theta^i\sim\mu^\infty(\theta|h^i)$  \Comment{effective fitness via sampling}
\State OR:
\State $\overline{F}^i \gets \overline{\F}(h^i)$ approximated via \eqref{eq:timefitness}  \Comment{effective fitness via time averaging}

\EndFor
    \For{$i = 1, \dots, N$}                      \Comment{population interaction} 
  \State with probability $\tau$:       
    \State \qquad select a NN $(\theta^j,h^j)$ proportionally to $\exp(\overline{F}^j)$   
    \State \qquad sample random mutation vector $\xi$
     \State \qquad $\tilde \theta^i \gets \theta^j$  \Comment{copy parameter}
    \State \qquad $\tilde h^i \gets h^j + \sigma\xi$   \Comment{copy and perturb hyperparameter}
    
 \EndFor
  \State $\theta^i\gets \tilde \theta^i, h^i \gets \tilde h^i$ for all $i=1,\dots,N$
  \State $t \gets t + \tau$
\EndWhile
\end{algorithmic}
\end{algorithm}

\subsubsection{Himmelblau function}

To study parameter distributions in a non-convex setting, we consider the two-dimensional Himmelblau function,
\begin{equation}
\mathcal{F}(\theta)=(\theta_0^2+\theta_1-11)^2+(\theta_0+\theta_1^2-7)^2, \quad \theta=(\theta_0, \theta_1)\in \RR^2,
\end{equation}
with, as before, a biased loss for $h_0\in \RR$
\begin{equation}
\mathcal{L}(\theta,h_0)=((\theta_0-h_0)^2+\theta_1-11)^2+(\theta_0+(\theta_1-h_0)^2-7)^2,
\end{equation}
The Himmelblau function has four distinct global minima, each attaining $\mathcal{F}=0$ (see Figure~\ref{fig:himmel}). Since our framework assumes maximisation of the fitness we will consider $-\F$ when evaluating the agents' fitness.  

Figure~\ref{fig:theta-dist} illustrates the underlying mechanism of Algorithm \ref{alg:pbt}: agents first approach nearby minima via SGD dynamics, after which copying high-performing agents can relocate particles into different basins. Subsequent SGD steps refine the parameters within these basins, ultimately leading to concentration around the a single optimum located at $(3,2)$.

\begin{figure}[t]
    \centering
    \begin{subfigure}[h!]{0.25\textwidth}
        \centering
        \includegraphics[width=\textwidth]{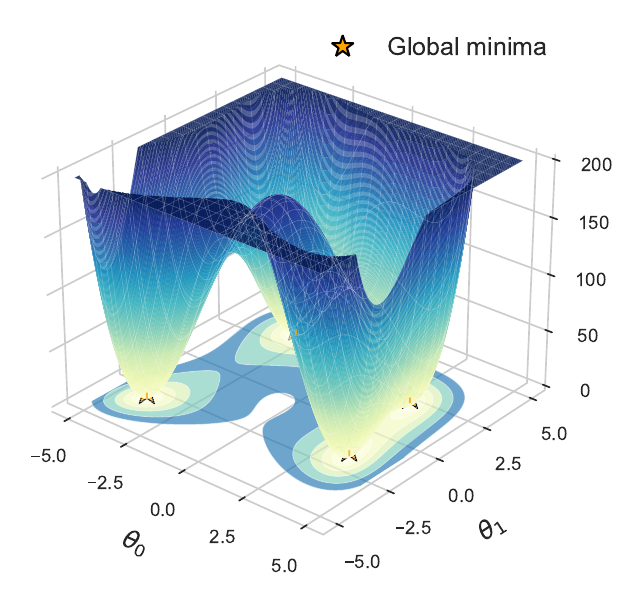}
        \caption{Himmelblau function}
        \label{fig:himmel}
    \end{subfigure}
    \hfill
    \begin{subfigure}[t]{0.73\textwidth}
        \centering
        \includegraphics[width=\textwidth]{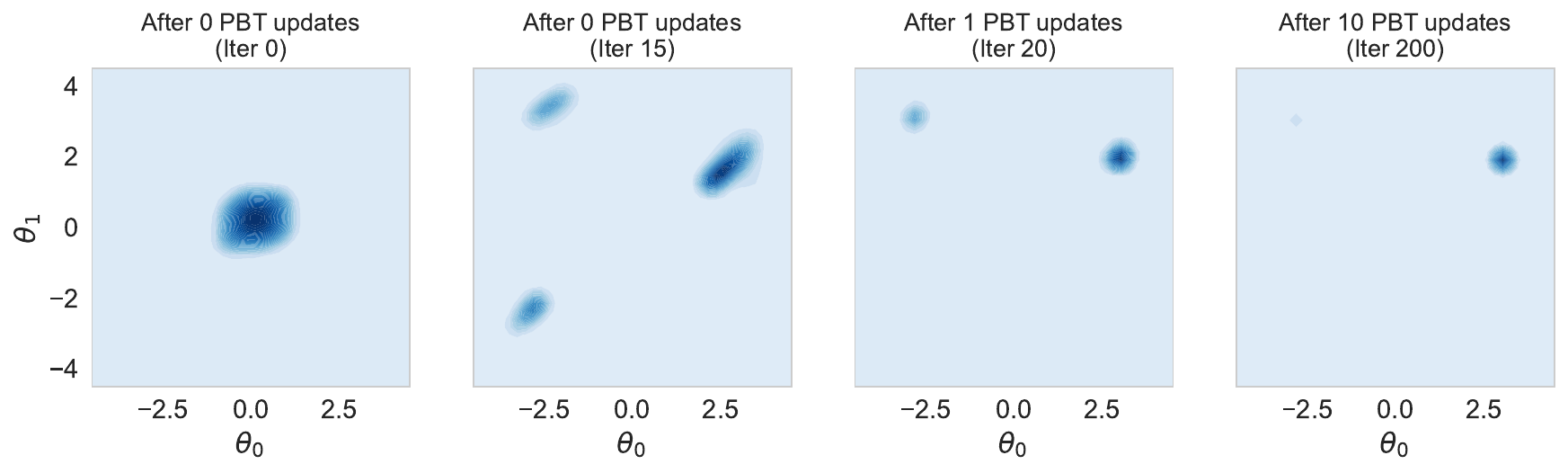} \\
        \includegraphics[width=\textwidth]{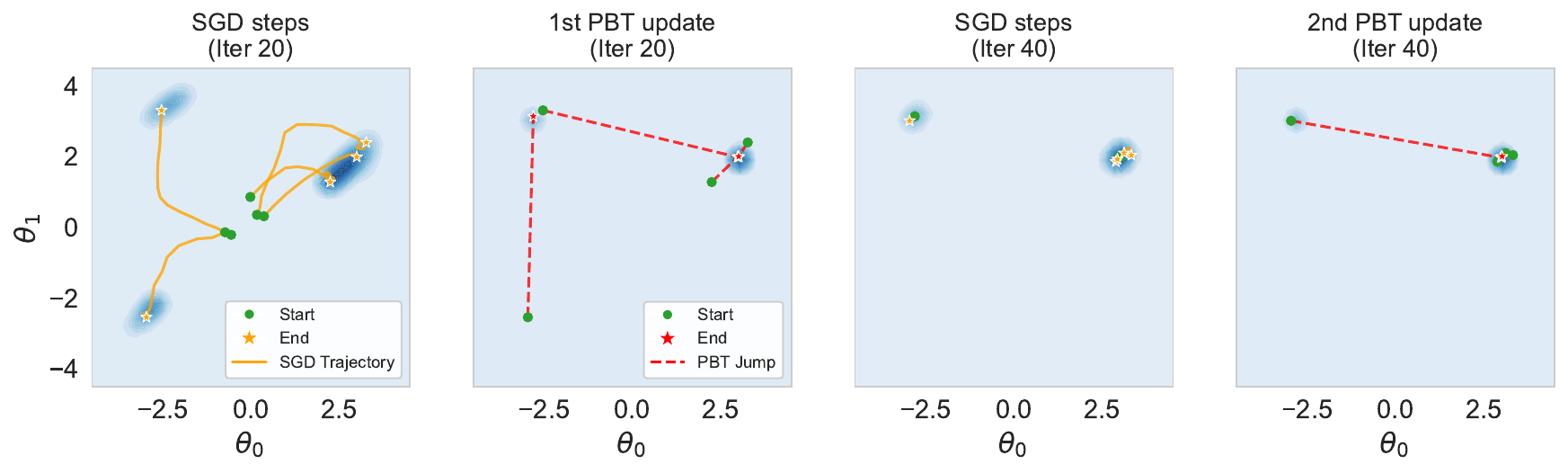}
        \caption{Evolution of the distribution of \(\theta\)}
        \label{fig:theta-dist}
    \end{subfigure}
    \caption{Himmelblau test, parameters' evolution. The population size is $N=10^5$, and the parameters $\theta$ are initially sampled uniformly from the box $[-0.5,0.5]^2$. The density is reconstructed using a 40-bin histogram in each dimension. The second row in (b) shows the trajectories of 5 randomly selected agents. The solid orange curves correspond to the SGD training dynamics, while the red dashed lines indicate the jumps induced by the PBT update after an inner training loop.}
    \label{fig:theta}
\end{figure}


\subsection{Deep Reinforcement Learning}

In this test, we show how the two-time-scale modelling framework extends beyond the simplified settings considered in our analysis. We consider a population of neural networks that learn collectively to control a  Cart-Pole system through the PBT algorithm. Training is carried out using the Adam optimiser, while the selection-mutation updates are performed via rank-based selection, in which the worst-performing neural networks are replaced with copies of the best-performing ones, see Section \ref{sec:extensions} for more details. We investigate the emergence of a two-time-scale separation by considering different numbers of inner training steps and by approximating the effective fitness through time averaging, as in \eqref{eq:timefitness}.


Controlling a Cart-Pole system is one of the classic benchmark problems in control theory and reinforcement learning \citep{barto1983Neuronlike,sutton1998Reinforcement}. In this environment, a pole is attached to a cart via an unactuated joint, and the cart can move along a frictionless track. The pole starts in an upright position, and the objective is to keep it balanced for as long as possible by applying forces to the cart either to the left or to the right.

The state at time steps $t = 0,1,2,\dots$ consists of four components:
\begin{equation*}
    s_t=(x_t, \dot{x}_t,\phi_t, \dot{\phi}_t),
\end{equation*}
which are, respectively, the cart position, cart velocity, pole angle, and pole angular velocity. The action space contains two discrete actions, $a_t$: left ($a_t = 0$) and right ($a_t = 1$). The cart position $x$ lies within the interval $(-4.8, 4.8)$, and an episode terminates early if the cart leaves the range $(-2.4, 2.4)$. Similarly, the pole angle $\phi$ is observable within $(-0.418, 0.418)$ radians, and the episode ends if $\phi \notin (-0.2095, 0.2095)$.

By default, the agent receives a reward $r_t=1$ at every time step, including the termination step, reflecting the objective of maximizing the duration for which the pole remains upright. An episode concludes either prematurely, when a failure condition is met (the pole falls or the cart goes out of bounds), or successfully, when the cumulative reward reaches a predefined maximum, which can be adjusted.

To investigate the two-time-scale dynamics of neural network training, we implement a Deep Q-Network (DQN) \citep{mnih2015Humanlevel}. The network architecture consists of an input layer with four state nodes, two hidden layers with 64 units each, and an output layer with two action nodes. Here, we use Adam \citep{kingma2014adam} for training, and PBT is employed to optimize three hyperparameters: the learning rate of the Adam optimiser, the greedy policy decay rate, and the experience replay batch size. To select the best agents to be copied, we use a rank-based strategy, also known as truncation selection (see Section~\ref{sec:extensions}). Agents are assigned a rank according to their fitness, and the top $20\%$ of the population is copied to replace the bottom $20\%$. Further algorithmic and hyperparameter details are provided in Appendix~\ref{app:numerics}.

\subsubsection{Results}
Our modelling analysis suggested that, in the presence of  a two-time-scale separation, the slow hyperparameter dynamics attempts to optimise a fitness that is averaged with respect to the parameter distribution. In this experiment, the fast time scale corresponds to the environment interactions (RL training steps), while the slow time scale corresponds to the evolutionary hyperparameter updates. Since sampling from equilibrium is not possible, we instead perform a time averaging of the fitness, which is here the episodic rewards.

We define a generation as a single PBT update interval consisting of $300, 500$, or $1000$ training steps. The maximum episodic reward is capped at $100$, so an agent completes at least $3,5$, and $10$ full episodes before evaluation, respectively. During the slow-scale update, an agent's fitness score is calculated as the moving average of its last $m$ episodic rewards. Figure~\ref{fig:dqn-tt} illustrates that increasing the moving-average window size $m$ effectively smooths the fast-scale performance noise. 

Figure~\ref{fig:dqn-h} shows that, after $30$ PBT updates, high-performing hyperparameter configurations concentrate in a localised region of the search space, supporting mean convergence toward the fittest solutions (Theorem~\ref{t:convergence}). The influence of the mutation noise $\sigma$ is also evident: larger values of $\sigma$ lead to a more dispersed distribution of hyperparameters across the population.


\begin{figure}[t]
    \centering
    \begin{subfigure}{0.32\textwidth}
        \centering
        \includegraphics[width=\textwidth]{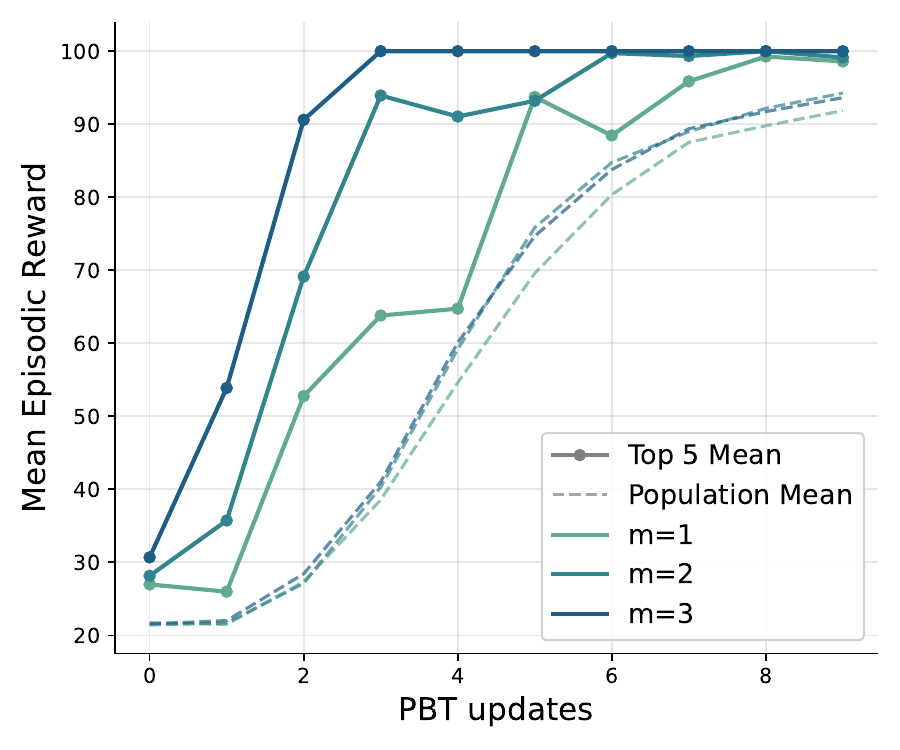}
        \caption{300 training steps}
    \end{subfigure}
    \hfill 
    \begin{subfigure}{0.32\textwidth}
    \centering
    \includegraphics[width=\textwidth]{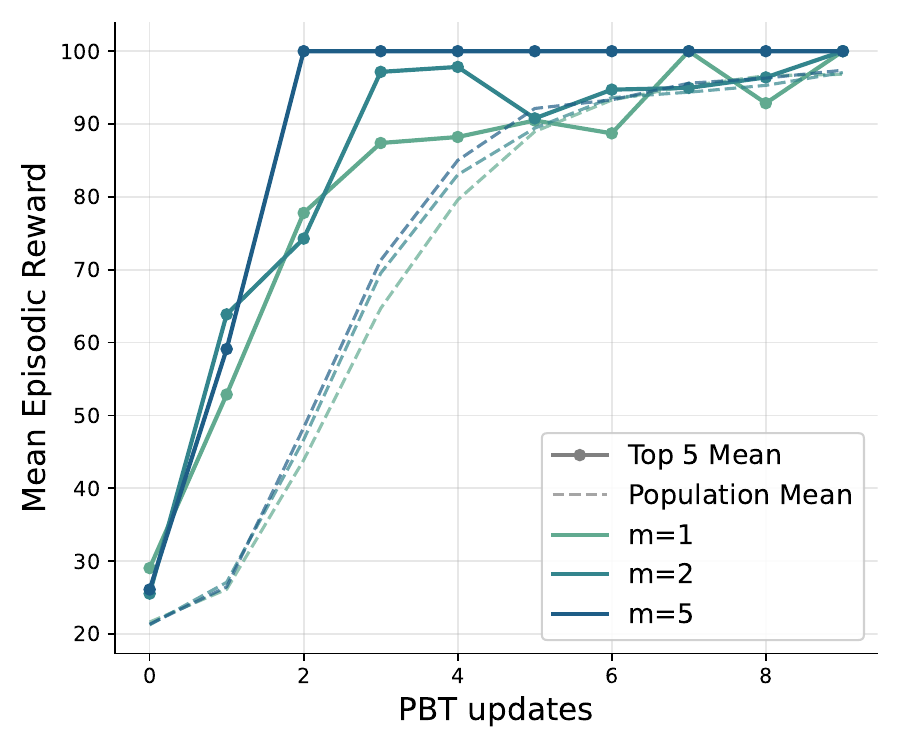}
    \caption{500 training steps}
    \end{subfigure}
    \hfill
    \begin{subfigure}{0.32\textwidth}
    \centering
    \includegraphics[width=\textwidth]{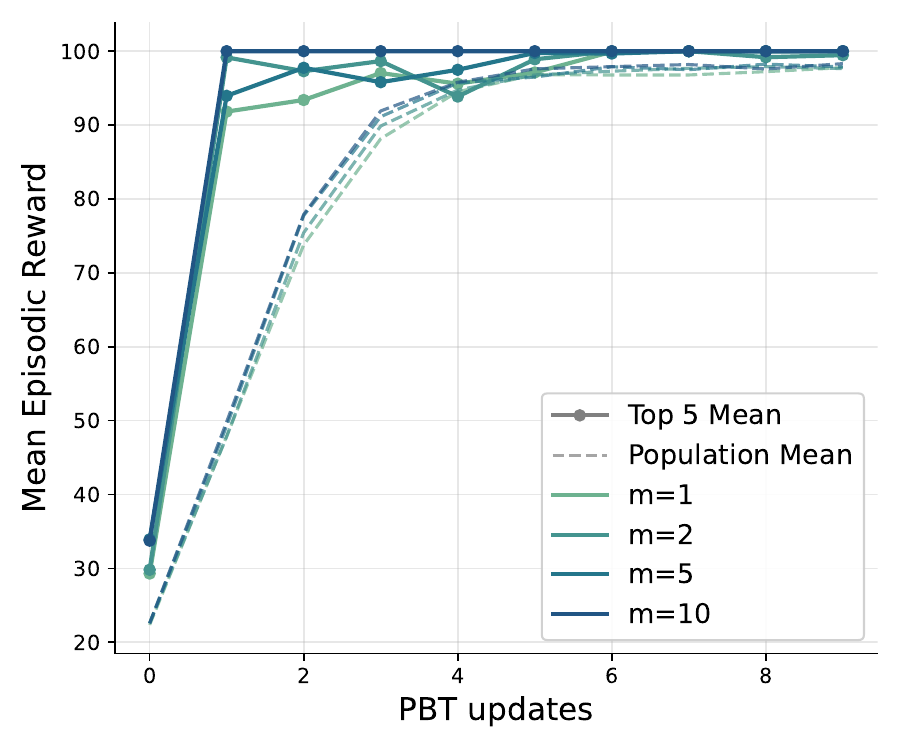}
    \caption{1000 training steps}  
    \end{subfigure}

    \caption{Cart–Pole control: two-time-scale dynamics. Mean episodic reward of top five agents (solid line with markers) and entire population (dashed line) across generations. We solve the RL problem with $N= 500$ NN agents via Population-Based Training evolution. 
    The parameter averaging window size indicates the number of recent episodes ($m$) averaged to compute the agent's fitness score. Larger window sizes yield more stable evaluations, reducing fast-scale noise and leading to faster convergence.}
    \label{fig:dqn-tt}
\end{figure}

\begin{figure}[t]
    \centering
    \begin{subfigure}{\textwidth}
    \centering
    \includegraphics[width=\textwidth]{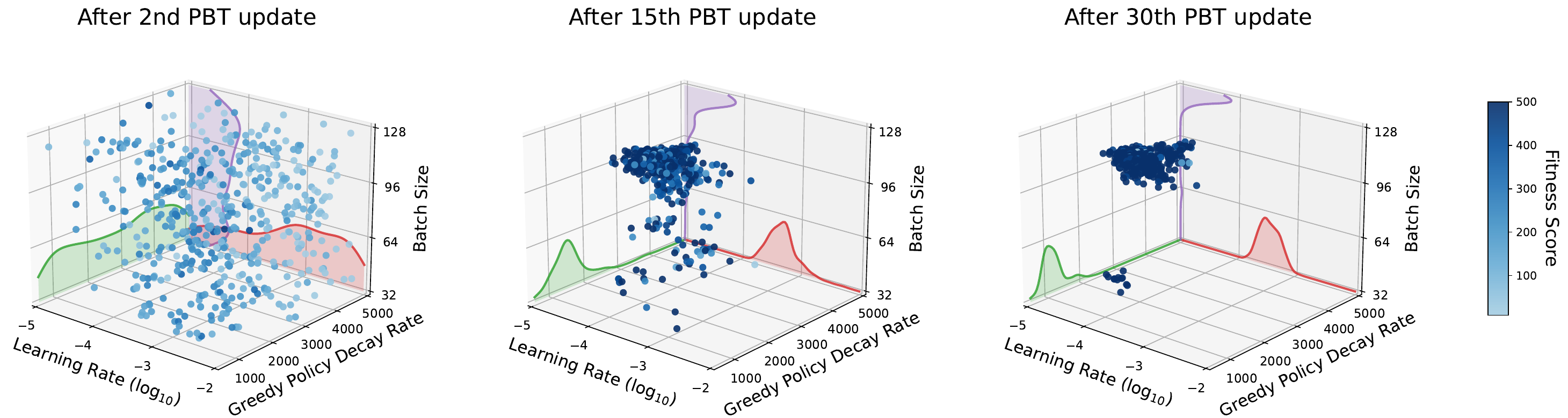}
    \caption{Low mutation noise, $\sigma=0.05$}
    \end{subfigure}
    \medskip
    
    \begin{subfigure}{\textwidth}
    \centering
    \includegraphics[width=\textwidth]{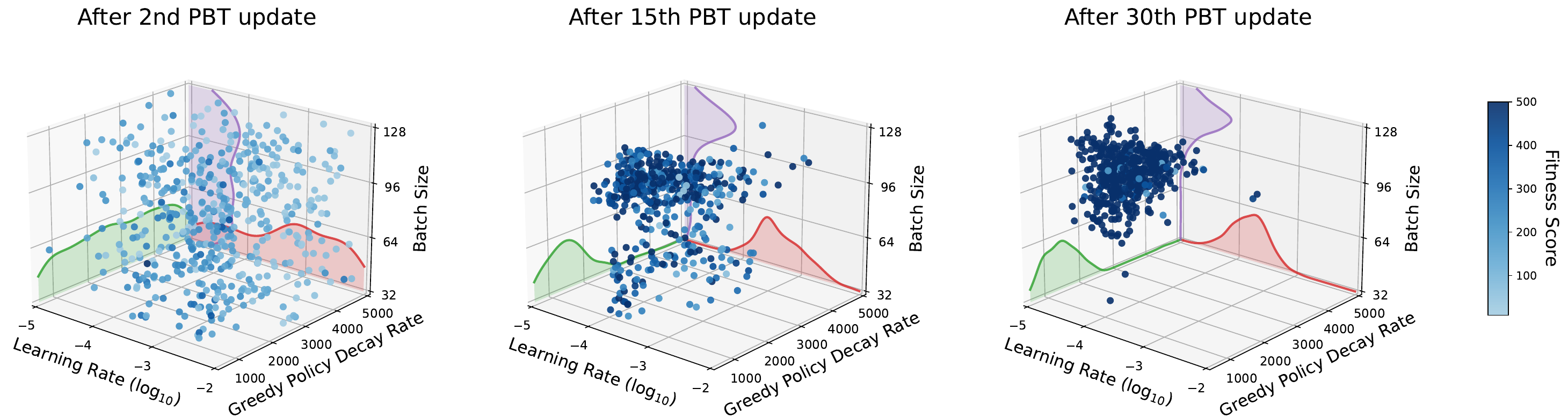}
    \caption{High mutation noise, $\sigma=0.1$}
    \end{subfigure}
    
    \caption{Cart--Pole control: hyperparameter distribution.
    For two different values of the mutation noise $\sigma$, we show the distribution of hyperparameters across $N=500$ agents, using $1000$ training steps and an averaging window of size $m=2$. The maximum episodic reward is $500$. The marginal distributions in each panel are estimated via Gaussian kernel density estimation.
    }
    \label{fig:dqn-h}
\end{figure}

\section{Conclusion and outlook} \label{sec:outlook}

The results of this work suggest a different way to understand, and possibly design, population-based neural training. By placing heuristic population methods within a continuous-time large-population framework, we move from viewing training as a collection of algorithmic rules toward a description in which learning is shaped by collective statistical dynamics. In the strong separation regime, the joint  evolution reduces to a closed selection--mutation equation for the hyperparameter density $\rho(t,h)$ driven by the Gibbs-averaged fitness $\overline{\F}(h)$. In this setting, population-level quantities, rather than individual optimisation trajectories, become the natural objects of analysis, comparison, and control.

A key consequence of this perspective is the emergence of an effective fitness as a macroscopic, statistically averaged quantity. In the asymptotic regime considered here, performance is no longer tied to the outcome of a single training run, but to properties of an equilibrium distribution induced by noisy learning dynamics. This suggests the possibility of surrogate training strategies in which explicit optimisation of individual networks is partially replaced by analytically or numerically accessible population-level criteria. While such ideas remain to be explored beyond the idealised regimes considered here, they point toward a principled route to reducing computational cost, improving stability, and increasing interpretability in complex training pipelines.

The framework developed here makes it possible to study population-based learning as a dynamical system, using questions familiar from kinetic theory and statistical physics, including stability, concentration, phase transitions, and the role of noise and interaction in shaping macroscopic behaviour. It also provides a natural basis for the systematic design of meta-learning and hyperparameter optimisation protocols, including optimal control formulations on the reduced dynamics.

Extending this approach to heterogeneous populations, structured interactions, crossover or model-merging mechanisms, finite-population effects, and non-equilibrium training scenarios is a natural direction for future research, with potential implications for distributed, federated, and large-scale ensemble learning.

The replicator–mutator structure of the averaged dynamics suggests that the effectiveness of PBT is not entirely heuristic: the convergence guarantee to the optimal hyperparameters configuration provides a principled basis for the selection pressure and mutation noise design choices made in practice.
Developing these ideas into a quantitative theory may help make population-based training a more systematic and interpretable component of modern machine learning.

\subsection*{Aknowledgments}

The work of Lorenzo Pareschi and Giacomo Borghi was supported by the Royal Society under the Wolfson Fellowship “Uncertainty quantification, data-driven simulations and learning of multiscale complex systems governed by PDEs". Lorenzo Pareschi also acknowledges the partial support of the FIS2023-01334 Advanced Grant “Tackling complexity: advanced numerical approaches for
multiscale systems with uncertainties" (ADAMUS).

\bibliography{ref}

@inproceedings{li2019generalized,
author = {Li, Ang and Spyra, Ola and Perel, Sagi and Dalibard, Valentin and Jaderberg, Max and Gu, Chenjie and Budden, David and Harley, Tim and Gupta, Pramod},
title = {A Generalized Framework for Population Based Training},
year = {2019},
isbn = {9781450362016},
publisher = {Association for Computing Machinery},
address = {New York, NY, USA},
url = {https://doi.org/10.1145/3292500.3330649},
doi = {10.1145/3292500.3330649},
booktitle = {Proceedings of the 25th ACM SIGKDD International Conference on Knowledge Discovery \& Data Mining},
pages = {1791–1799},
numpages = {9},
location = {Anchorage, AK, USA},
series = {KDD '19}
}

@article{lu2023birth,
  title={Birth--death dynamics for sampling: global convergence, approximations and their asymptotics},
  author={Lu, Yulong and Slep{\v{c}}ev, Dejan and Wang, Lihan},
  journal={Nonlinearity},
  volume={36},
  number={11},
  pages={5731--5772},
  year={2023},
  publisher={IOP Publishing}
}

@article{gallouet2019unbalanced,
  title={An unbalanced optimal transport splitting scheme for general advection-reaction-diffusion problems},
  author={Gallou{\"e}t, Thomas and Laborde, Maxime and Monsaingeon, Leonard},
  journal={ESAIM: Control, Optimisation and Calculus of Variations},
  volume={25},
  pages={8},
  year={2019},
  publisher={EDP Sciences}
}

@incollection{bobylev2020chapter7,
  author    = {Bobylev, Alexander V.},
  title     = {Boltzmann Equation and Hydrodynamics beyond {N}avier---{S}tokes},
  booktitle = {Boltzmann Equation, {M}axwell Models, and Hydrodynamics beyond {N}avier--{S}tokes},
  chapter   = {7},
  pages     = {195--234},
  publisher = {De Gruyter},
  address   = {Berlin and Boston},
  year      = {2020},
  doi       = {10.1515/9783110550986-008},
  url       = {https://doi.org/10.1515/9783110550986-008},
  volume = {1},
  isbn      = {9783110550986},
  urldate   = {2026-03-01}
}

@inproceedings{delmoral2002stability,
  title={On the stability of nonlinear {F}eynman-{K}ac semigroups},
  author={Del Moral, Pierre and Miclo, Laurent},
  booktitle={Annales de la Facult{\'e} des sciences de Toulouse: Math{\'e}matiques},
  volume={11},
  number={2},
  pages={135--175},
  year={2002}
}

@article{cerf1998ga, title={Asymptotic convergence of genetic algorithms}, volume={30}, DOI={10.1239/aap/1035228082}, number={2}, journal={Advances in Applied Probability}, author={Cerf, Raphaël}, year={1998}, pages={521–550}}

@Inbook{pavliotis2014chapter4,
author="Pavliotis, Grigorios A.",
title="The {F}okker--{P}lanck Equation",
bookTitle="Stochastic Processes and Applications: Diffusion Processes, the {F}okker--{P}lanck and {L}angevin Equations",
year="2014",
publisher="Springer New York",
address="New York, NY",
pages="87--137",
isbn="978-1-4939-1323-7",
doi="10.1007/978-1-4939-1323-7_4",
url="https://doi.org/10.1007/978-1-4939-1323-7_4"
}

@article{liero2018optimal, 
  title={Optimal entropy-transport problems and a new Hellinger--Kantorovich distance between positive measures},
  author={Liero, Matthias and Mielke, Alexander and Savar{\'e}, Giuseppe},
  journal={Inventiones mathematicae},
  volume={211},
  number={3},
  pages={969--1117},
  year={2018},
  publisher={Springer}
}

@book{bogachev2022fokker,
  title={{F}okker--{P}lanck--{K}olmogorov Equations},
  author={Bogachev, Vladimir I and Krylov, Nicolai V and R{\"o}ckner, Michael and Shaposhnikov, Stanislav V},
  volume={207},
  year={2022},
  publisher={American Mathematical Society}
}

@article{demircigil2025convergence,
  title={Convergence and Wave Propagation for a System of Branching Rank-Based Interacting {B}rownian Particles},
  author={Demircigil, Mete and Tomasevic, Milica},
  journal={arXiv preprint arXiv:2505.08563},
  year={2025}
}

@article{bench2022weak,
author = {Oumaima Bencheikh and Benjamin Jourdain},
title = {Weak and strong error analysis for mean-field rank-based particle approximations of one-dimensional viscous scalar conservation laws},
volume = {32},
journal = {The Annals of Applied Probability},
number = {6},
publisher = {Institute of Mathematical Statistics},
pages = {4143 -- 4185},
year = {2022},
doi = {10.1214/21-AAP1776},
URL = {https://doi.org/10.1214/21-AAP1776}
}

@article{barto1983Neuronlike,
  title = {Neuronlike Adaptive Elements That Can Solve Difficult Learning Control Problems},
  author = {Barto, Andrew G. and Sutton, Richard S. and Anderson, Charles W.},
  year = {1983},
  journaltitle = {IEEE Transactions on Systems, Man, and Cybernetics},
  shortjournal = {IEEE Trans. Syst., Man, Cybern.},
  volume = {SMC-13},
  number = {5},
  pages = {834--846},
  issn = {0018-9472, 2168-2909},
  doi = {10.1109/TSMC.1983.6313077},
  url = {http://ieeexplore.ieee.org/document/6313077/},
  urldate = {2025-12-09}
}

@article{LuTadmorZenginoglu2024Swarm,
  title   = {Swarm-Based Gradient Descent Method for Non-Convex Optimization},
  author  = {Lu, Jingcheng and Tadmor, Eitan and Zenginoğlu, Anil},
  journal = {Communications of the American Mathematical Society},
  volume  = {4},
  year    = {2024},
  pages   = {159--195},
  doi     = {10.1090/cams/30}
}

@book {villani2009,
    AUTHOR = {Villani, C\'edric},
     TITLE = {Optimal transport},
    SERIES = {Grundlehren der mathematischen Wissenschaften [Fundamental
              Principles of Mathematical Sciences]},
    VOLUME = {338},
      NOTE = {Old and new},
 PUBLISHER = {Springer-Verlag, Berlin},
      YEAR = {2009},
     PAGES = {xxii+973},
      ISBN = {978-3-540-71049-3},
   MRCLASS = {49-02 (28A75 37J50 49Q20 53C23 58E30)},
  MRNUMBER = {2459454},
MRREVIEWER = {Dario\ Cordero-Erausquin},
       DOI = {10.1007/978-3-540-71050-9},
}

@article{Mignacco2025,
  title          = {A statistical physics framework for optimal learning},
  author         = {Francesca Mignacco and Francesco Mori},
  journal        = {arXiv:2507.07907},
  year           = {2025},
}

@article{mnih2015Humanlevel,
  title = {Human-Level Control through Deep Reinforcement Learning},
  author = {Mnih, Volodymyr and Kavukcuoglu, Koray and Silver, David and Rusu, Andrei A. and Veness, Joel and Bellemare, Marc G. and Graves, Alex and Riedmiller, Martin and Fidjeland, Andreas K. and Ostrovski, Georg and Petersen, Stig and Beattie, Charles and Sadik, Amir and Antonoglou, Ioannis and King, Helen and Kumaran, Dharshan and Wierstra, Daan and Legg, Shane and Hassabis, Demis},
  year = {2015},
  journaltitle = {Nature},
  shortjournal = {Nature},
  volume = {518},
  number = {7540},
  pages = {529--533},
  issn = {0028-0836, 1476-4687},
  doi = {10.1038/nature14236},
  url = {https://www.nature.com/articles/nature14236},
  urldate = {2025-12-09},
  langid = {english}
}

@article{cui2018evolutionary,
  title={Evolutionary stochastic gradient descent for optimization of deep neural networks},
  author={Cui, Xiaodong and Zhang, Wei and T{\"u}ske, Zolt{\'a}n and Picheny, Michael},
  journal={Advances in neural information processing systems},
  volume={31},
  year={2018}
}

@article{watkins1992Qlearning,
  title = {Q-Learning},
  author = {Watkins, Christopher J. C. H. and Dayan, Peter},
  date = {1992-05},
  journaltitle = {Machine Learning},
  shortjournal = {Mach Learn},
  volume = {8},
  number = {3--4},
  pages = {279--292},
  issn = {0885-6125, 1573-0565},
  doi = {10.1007/BF00992698},
  url = {http://link.springer.com/10.1007/BF00992698},
  urldate = {2025-11-18},
  langid = {english}
}

@book{sutton1998Reinforcement,
  title = {Reinforcement Learning: An Introduction},
  shorttitle = {Reinforcement Learning},
  author = {Sutton, Richard S. and Barto, Andrew G.},
  series = {Adaptive Computation and Machine Learning},
  publisher = {MIT Press},
  location = {Cambridge, Mass},
  isbn = {978-0-262-19398-6},
  langid = {english},
  pagetotal = {322},
  year = {1998}
}

@article{jaderberg2017Population,
  title={Population Based Training of Neural Networks},
  author={Jaderberg, Max and Dalibard, Valentin and Osindero, Simon and Czarnecki, Wojciech M and Donahue, Jeff and Razavi, Ali and Vinyals, Oriol and Green, Tim and Dunning, Iain and Simonyan, Karen and Fernando, Chrisantha and Kavukcuoglu, Koray},
  journal={Preprint arXiv:1711.09846},
  year={2017}
}

@article{mandt2017sgd,
  title={Stochastic gradient descent as approximate Bayesian inference},
  author={Mandt, Stephan and Hoffman, Matthew D and Blei, David M},
  journal={Journal of Machine Learning Research},
  volume={18},
  number={134},
  pages={1--35},
  year={2017}
}

@article{li2019stochastic,
  title={Stochastic Modified Equations and Dynamics of Stochastic Gradient Algorithms I: Mathematical Foundations},
  author={Li, Qianxiao and Tai, Kaiqi and E, Weinan},
  journal={Journal of Machine Learning Research},
  volume={20},
  number={40},
  pages={1--47},
  year={2019}
}

@article{jumper2021highly,
  title   = {Highly accurate protein structure prediction with {AlphaFold}},
  author  = {Jumper, John and Evans, Richard and Pritzel, Alexander and Green, Tim and Figurnov, Michael and Ronneberger, Olaf and Tunyasuvunakool, Kathryn and Bates, Russ and {\v{Z}}{\'\i}dek, Augustin and Potapenko, Anna and others},
  journal = {Nature},
  volume  = {596},
  number  = {7873},
  pages   = {583--589},
  year    = {2021},
  publisher = {Nature Publishing Group}
}

@article{ottobre2011asymptotic,
  title={Asymptotic analysis for the generalized {L}angevin equation},
  author={Ottobre, Michela and Pavliotis, Grigorios A},
  journal={Nonlinearity},
  volume={24},
  number={5},
  pages={1629},
  year={2011},
  publisher={IOP Publishing}
}

@inproceedings{chaudhari2018sgd,
  title={Stochastic Gradient Descent Performs Variational Inference, Converges to Limit Cycles and Fixed Points},
  author={Chaudhari, Pratik and Soatto, Stefano},
  booktitle={Advances in Neural Information Processing Systems},
  volume={31},
  year={2018}
}

@article{alfaro2019evolutionary,
  title={Evolutionary branching via replicator--mutator equations},
  author={Alfaro, Matthieu and Veruete, Mario},
  journal={Journal of Dynamics and Differential Equations},
  volume={31},
  number={4},
  pages={2029--2052},
  year={2019},
  publisher={Springer}
}

@article{hasenpflug2024wasserstein,
  title={Wasserstein convergence rates of increasingly concentrating probability measures},
  author={Hasenpflug, Mareike and Rudolf, Daniel and Sprungk, Bj{\"o}rn},
  journal={The Annals of Applied Probability},
  volume={34},
  number={3},
  pages={3320--3347},
  year={2024},
  publisher={Institute of Mathematical Statistics}
}

@inproceedings{wright2023fully,
  title={A fully first-order method for stochastic bilevel optimization},
  author={Kwon, Jeongyeol and Kwon, Dohyun and Wright, Stephen and Nowak, Robert D},
  booktitle={International Conference on Machine Learning},
  pages={18083--18113},
  year={2023},
  organization={PMLR}
}

@incollection{sznitman1991topics,
	Author = {Sznitman, Alain-Sol},
	Booktitle = {Ecole d'\'{e}t\'{e} de probabilit\'{e}s de {Saint-Flour} {XIX}---1989},
	Pages = {165--251},
	Publisher = {Springer},
	Title = {Topics in propagation of chaos},
	Year = {1991}}

@book{delmoral2004book,
  title={Feynman-Kac formulae: genealogical and interacting particle systems with applications},
  author={Del Moral, Pierre},
  year={2004},
  publisher={Springer},
doi = {10.1007/978-1-4684-9393-1}
}

@article{ackleh2016population,
title = {Population dynamics under selection and mutation: Long-time behavior for differential equations in measure spaces},
journal = {Journal of Differential Equations},
volume = {261},
number = {2},
pages = {1472-1505},
year = {2016},
issn = {0022-0396},
doi = {https://doi.org/10.1016/j.jde.2016.04.008},
url = {https://www.sciencedirect.com/science/article/pii/S002203961630033X},
author = {Azmy S. Ackleh and John Cleveland and Horst R. Thieme}
}

@book{dembo2010,
	Author = {Dembo, Amir and Zeitouni, Ofer},
	Publisher = {Springer-Verlag Berlin Heidelberg},
	Title = {Large Deviations Techniques and Applications},
	Volumes = {38},
	Year = {2010}}

@inproceedings{bergstra2011Algorithms,
  title = {Algorithms for Hyper-Parameter Optimization},
  booktitle = {Advances in Neural Information Processing Systems},
  author = {Bergstra, James and Bardenet, Rémi and Bengio, Yoshua and Kégl, Balázs},
  editor = {Shawe-Taylor, J. and Zemel, R. and Bartlett, P. and Pereira, F. and Weinberger, K.Q.},
  year = {2011},
  volume = {24},
  publisher = {Curran Associates, Inc.},
  url = {https://proceedings.neurips.cc/paper_files/paper/2011/file/86e8f7ab32cfd12577bc2619bc635690-Paper.pdf}
}

@book{perthame2007transport,
  title={Transport equations in biology},
  author={Perthame, Beno{\^\i}t},
  year={2007},
  publisher={Springer}
}

@article{borghi2025kinetic,
  title = {Kinetic Description and Convergence Analysis of Genetic Algorithms for Global Optimization},
  author = {Borghi, Giacomo and Pareschi, Lorenzo},
  year = {2025},
  journal = {Communications in Mathematical Sciences},
  volume = {23},
  number = {3},
  pages = {641--668},
  issn = {15396746, 19450796},
  doi = {10.4310/CMS.250208214404},
  url = {https://link.intlpress.com/JDetail/1888222908206624770},
  urldate = {2025-05-08},
  langid = {english}
}

@article{fornasier2021cbo,
  author    = {Fornasier, Massimo and Huang, Hui and Pareschi, Lorenzo and S{\"u}nnen, Philippe},
  title     = {Consensus-based Optimization on the Sphere: Convergence to Global Minimizers and Machine Learning},
  journal   = {Journal of Machine Learning Research},
  volume    = {22},
  number    = {237},
  pages     = {1--55},
  year      = {2021}
}

@article{fornasier2024consensus,
  title={Consensus-based optimization methods converge globally},
  author={Fornasier, Massimo and Klock, Timo and Riedl, Konstantin},
  journal={SIAM Journal on Optimization},
  volume={34},
  number={3},
  pages={2973--3004},
  year={2024},
  publisher={SIAM}
}

@article{carrillo2024fedcbo,
  title={Fed{CBO}: {R}eaching group consensus in clustered federated learning through consensus-based optimization},
  author={Carrillo, Jos{\'e} A and Trillos, Nicolas Garcia and Li, Shiyi and Zhu, Yifei},
  journal={Journal of Machine Learning Research},
  volume={25},
  number={214},
  pages={1--51},
  year={2024}
}

@article{borghi2024unified,
      title={Kinetic models for optimization: a unified mathematical framework for metaheuristics}, 
      author={Giacomo Borghi and Michael Herty and Lorenzo Pareschi},
      year={2024},
      eprint={2410.10369},
      journal={Preprint arXiv:2410.10369},
      primaryClass={math.OC},
      url={https://arxiv.org/abs/2410.10369}, 
}

@article{trillos2024CB2O,
  title={{CB}$^2${O}: {{Consensus-Based Bi-Level Optimization}}},
  author={Trillos, Nicol{\'a}s and Li, Sixu and Riedl, Konstantin and Zhu, Yuhua},
  journal={Preprint arXiv:2411.13394},
  year={2024},
  url={https://arxiv.org/abs/2411.13394}
}

@inproceedings{franceschi2018Bilevel,
  title={Bilevel Programming for Hyperparameter Optimization and Meta-Learning},
  author={Franceschi, Luca and Frasconi, Paolo and Salzo, Saverio and Grazzi, Riccardo and Pontil, Massimiliano},
  booktitle={Proceedings of the 35th International Conference on Machine Learning},
  pages={1568--1577},
  year={2018},
  editor={Dy, Jennifer and Krause, Andreas},
  volume={80},
  series={Proceedings of Machine Learning Research},
  publisher={PMLR}
}

@article{bergstra2012Random,
  title = {Random Search for Hyper-Parameter Optimization},
  author = {Bergstra, James and Bengio, Yoshua},
  year = 2012,
  journal = {Journal of Machine Learning Research},
  volume = {13},
  number = {10},
  pages = {281--305}
}

@inproceedings{
li2022what,
title={What Happens after {SGD} Reaches Zero Loss? --A Mathematical Framework},
author={Zhiyuan Li and Tianhao Wang and Sanjeev Arora},
booktitle={International Conference on Learning Representations},
year={2022},
url={https://openreview.net/forum?id=siCt4xZn5Ve}
}

@article{achleitner2015FP,
  author  = {Achleitner, Franz and Arnold, Anton and St{\"u}rzer, Dominik},
  title   = {Large-time behavior in non-symmetric {F}okker--{P}lanck equations},
  journal = {Rivista di Matematica dell'Universit{\`a} di Parma},
  volume  = {6},
  number  = {1},
  year    = {2015},
  pages   = {1--68},
}

@article{khalid2013selection,
author = {Jebari, Khalid},
year = {2013},
month = {12},
pages = {333-344},
title = {Selection Methods for Genetic Algorithms},
volume = {3},
journal = {International Journal of Emerging Sciences}
}

@article{malladi2022sdes,
  title={On the {SDE}s and scaling rules for adaptive gradient algorithms},
  author={Malladi, Sadhika and Lyu, Kaifeng and Panigrahi, Abhishek and Arora, Sanjeev},
  journal={Advances in Neural Information Processing Systems},
  volume={35},
  pages={7697--7711},
  year={2022}
}

@inproceedings{cheng2020stochastic,
  title={Stochastic gradient and {L}angevin processes},
  author={Cheng, Xiang and Yin, Dong and Bartlett, Peter and Jordan, Michael},
  booktitle={International Conference on Machine Learning},
  pages={1810--1819},
  year={2020},
  organization={PMLR}
}

@article{li2021validity,
  title={On the validity of modeling {SGD} with stochastic differential equations ({SDE}s)},
  author={Li, Zhiyuan and Malladi, Sadhika and Arora, Sanjeev},
  journal={Advances in Neural Information Processing Systems},
  volume={34},
  pages={12712--12725},
  year={2021}
}

@article{mei2018mean,
  title={A mean field view of the landscape of two-layer neural networks},
  author={Mei, Song and Montanari, Andrea and Nguyen, Phan-Minh},
  journal={Proceedings of the National Academy of Sciences},
  volume={115},
  number={33},
  pages={E7665--E7671},
  year={2018},
  publisher={National Academy of Sciences}
}

@article{sirignano2020mean,
  title={Mean field analysis of neural networks: A central limit theorem},
  author={Sirignano, Justin and Spiliopoulos, Konstantinos},
  journal={Stochastic Processes and their Applications},
  volume={130},
  number={3},
  pages={1820--1852},
  year={2020},
  publisher={Elsevier}
}

@misc{suzuki2023adam,
      title={Adam-like Algorithm with Smooth Clipping Attains Global Minima: Analysis Based on Ergodicity of Functional SDEs}, 
      author={Keisuke Suzuki},
      year={2023},
      eprint={2312.02182},
      archivePrefix={arXiv},
      primaryClass={cs.LG},
      url={https://arxiv.org/abs/2312.02182}, 
}

@article{zito2025Metaheuristics,
  title = {Metaheuristics in Automated Machine Learning: {{Strategies}} for Optimization},
  shorttitle = {Metaheuristics in Automated Machine Learning},
  author = {Zito, Francesco and Talbi, El-Ghazali and Cavallaro, Claudia and Cutello, Vincenzo and Pavone, Mario},
  year = 2025,
  month = jun,
  journal = {Intelligent Systems with Applications},
  volume = {26},
  pages = {200532},
  issn = {26673053},
  doi = {10.1016/j.iswa.2025.200532},
  urldate = {2025-11-07},
  langid = {english}
}

@book{duong2020mean,
  title={Mean Field Type Control Theory},
  author={Duong, Minh H and Peletier, Mark A and Zimmer, Johannes},
  publisher={Springer},
  year={2020}
}

@article{bouteiller2026sociodynamics,
title={Sociodynamics of Reinforcement Learning},
author={Yann Bouteiller and Karthik Soma and Giovanni Beltrame},
journal={Transactions on Machine Learning Research},
issn={2835-8856},
year={2026},
url={https://openreview.net/forum?id=Ro6Ylnx8se}
}

@article{borghi2026chaos,
author = {Borghi, Giacomo},
title = {Chaos propagation in genetic algorithms: {A}n optimal transport approach},
journal = {Bulletin of the London Mathematical Society},
volume = {58},
number = {3},
pages = {e70333},
doi = {10.1112/blms.70333},
url = {https://londmathsoc.onlinelibrary.wiley.com/doi/abs/10.1112/blms.70333},
eprint = {https://londmathsoc.onlinelibrary.wiley.com/doi/pdf/10.1112/blms.70333},
year = {2026}
}

@book{PareschiToscani2013,
  author    = {Pareschi, Lorenzo and Toscani, Giuseppe},
  title     = {Interacting Multiagent Systems: Kinetic Equations 
               and Monte Carlo Methods},
  publisher = {Oxford University Press},
  address   = {Oxford},
  year      = {2013},
  isbn      = {978-0-19-966551-2}
}

@article{fournier2015rate,
	Author = {Fournier, Nicolas and Guillin, Arnaud},
	Journal = {Probability Theory and Related Fields},
	Number = {3-4},
	Pages = {707--738},
	Publisher = {Springer},
	Title = {On the rate of convergence in {W}asserstein distance of the empirical measure},
	Volume = {162},
	Year = {2015}}

@article{piccoli2016properties,
  title={On properties of the generalized {W}asserstein distance},
  author={Piccoli, Benedetto and Rossi, Francesco},
  journal={Archive for Rational Mechanics and Analysis},
  volume={222},
  number={3},
  pages={1339--1365},
  year={2016},
  publisher={Springer}
}

@articleInfo{diez2022chaos2,
title = {Propagation of chaos: A review of models, methods and applications. {II}. Applications},
journal = {Kinetic and Related Models},
volume = {15},
number = {6},
pages = {1017-1173},
year = {2022},
issn = {1937-5093},
doi = {10.3934/krm.2022018},
url = {https://www.aimsciences.org/article/id/62948a7a2d80b70dfd2582aa},
author = {Louis-Pierre Chaintron and Antoine Diez}
}

@article{haasdijk2018pressure,
  title={Quantifying selection pressure},
  author={Haasdijk, Evert and Heinerman, Jacqueline},
  journal={Evolutionary computation},
  volume={26},
  number={2},
  pages={213--235},
  year={2018},
  publisher={MIT Press}
}

@articleInfo{diez2022chaos1,
title = {Propagation of chaos: A review of models, methods and applications. {I}. Models and methods},
journal = {Kinetic and Related Models},
volume = {15},
number = {6},
pages = {895-1015},
year = {2022},
issn = {1937-5093},
doi = {10.3934/krm.2022017},
url = {https://www.aimsciences.org/article/id/631fd3b64cedfd0007ce7600},
author = {Louis-Pierre Chaintron and Antoine Diez}
}

@book{mao2011sde,
  title     = {Stochastic Differential Equations and Applications},
  author    = {Mao, Xuerong},
  year      = {2011},
  edition   = {Second Edition},
  publisher = {Woodhead Publishing},
  isbn      = {978-1-904275-34-3},
  doi       = {10.1533/9780857099402},
  url       = {https://www.sciencedirect.com/book/9781904275343/stochastic-differential-equations-and-applications}
}

@article{hairer2008spectral,
author = {Martin Hairer and Jonathan C. Mattingly},
title = {{Spectral gaps in {W}asserstein distances and the {2D} stochastic {N}avier–{S}tokes equations}},
volume = {36},
journal = {The Annals of Probability},
number = {6},
publisher = {Institute of Mathematical Statistics},
pages = {2050 -- 2091},
year = {2008},
doi = {10.1214/08-AOP392},
URL = {https://doi.org/10.1214/08-AOP392}
}

@article{salimans2017Evolution,
  title={Evolution Strategies as a Scalable Alternative to Reinforcement Learning},
  author={Salimans, Tim and Ho, Jonathan and Chen, Xi and Sutskever, Ilya},
  journal={Preprint arXiv:1703.03864},
  year={2017}
}

@incollection{hutter2011Sequential,
  title = {Sequential {{Model-Based Optimization}} for {{General Algorithm Configuration}}},
  booktitle = {Learning and {{Intelligent Optimization}}},
  author = {Hutter, Frank and Hoos, Holger H. and Leyton-Brown, Kevin},
  editor = {Coello, Carlos A. Coello},
  year = {2011},
  volume = {6683},
  pages = {507--523},
  publisher = {Springer Berlin Heidelberg},
  location = {Berlin, Heidelberg},
  doi = {10.1007/978-3-642-25566-3_40},
  url = {http://link.springer.com/10.1007/978-3-642-25566-3_40},
  urldate = {2025-12-01},
  isbn = {978-3-642-25565-6 978-3-642-25566-3}
}

@article{kingma2014adam,
  title={Adam: {A} method for stochastic optimization},
  author={Kingma, Diederik P and Ba, Jimmy},
  journal={arXiv preprint arXiv:1412.6980},
  year={2014}
}

@article{akiba2025evolutionary,
  title   = {Evolutionary Optimization of Model Merging Recipes},
  author  = {Akiba, Takuya and Shing, Miqing and Tang, Yanping and Sun, Qi and Ha, David},
  journal = {Nature Machine Intelligence},
  volume  = {7},
  year    = {2025},
pages={195–-204},
}

@article{Chizat2022,
  author  = {Chizat, Lenaic},
  title   = {Mean-field {L}angevin dynamics: Exponential convergence 
             and annealing},
  journal = {Transactions on Machine Learning Research},
  volume  = {8},
  year    = {2022},
  url     = {https://openreview.net/forum?id=0syRaGFIuO}
}

@article{carrillo2021cbo,
  title   = {A Consensus-Based Global Optimization Method for High-Dimensional Machine Learning Problems},
  author  = {Carrillo, José A. and Jin, Shi and Li, Lei and Zhu, Yuhua},
  journal = {ESAIM: Control, Optimisation and Calculus of Variations},
  volume  = {27},
  pages   = {S5},
  year    = {2021},
  doi     = {10.1051/cocv/2020046}
}

@article{Pareschi2024,
  author  = {Pareschi, Lorenzo},
  title   = {Optimization by linear kinetic equations and 
             mean-field {L}angevin dynamics},
  journal = {Mathematical Models and Methods in Applied Sciences},
  volume  = {34},
  number  = {12},
  pages   = {2191--2216},
  year    = {2024},
  doi     = {10.1142/S0218202524500428}
}

@article{benfenati2022binary,
  title   = {Binary Interaction Methods for High Dimensional Global Optimization and Machine Learning},
  author  = {Benfenati, Alessandro and Borghi, Giacomo and Pareschi, Lorenzo},
  journal = {Applied Mathematics \& Optimization},
  volume  = {86},
  number  = {9},
  pages   = {1--41},
  year    = {2022},
  doi     = {10.1007/s00245-022-09898-1},
  publisher = {Springer}
}

@inproceedings{wortsman2022model,
  title     = {Model Soups: Averaging Weights of Multiple Fine-Tuned Models Improves Accuracy Without Increasing Inference Time},
  author    = {Wortsman, Mitchell and Ilharco, Gabriel and Kim, Jong Wook and Li, Yi and Kornblith, Simon and Roelofs, Rebecca and Menon, Aditya K and Koumoutsakos, Petros and Farhadi, Ali and Yosinski, Jason and Carmon, Yair and Schmidt, Ludwig},
  booktitle = {Proceedings of the 39th International Conference on Machine Learning},
  pages     = {23965--23998},
  year      = {2022},
  editor    = {Chaudhuri, Kamalika and Jegelka, Stefanie and Song, Shuang and Banerjee, Arindam and Yilmaz, Yasin and Hsieh, Cho-Jui and Kumar, Aarti},
  volume    = {162},
  series    = {Proceedings of Machine Learning Research},
  publisher = {PMLR},
  url       = {https://proceedings.mlr.press/v162/wortsman22a.html}
}

@book{dudley2018real,
  title={Real analysis and probability},
  author={Dudley, Richard M},
  year={2018},
  publisher={Chapman and Hall/CRC}
}

@article{vinyals2019alphastar,
  title   = {Grandmaster level in {StarCraft II} using multi-agent reinforcement learning},
  author  = {Vinyals, Oriol and Babuschkin, Igor and Czarnecki, Wojciech M. and Mathieu, Micha{\"e}l and Georgiev, Petko and Oh, Junhyuk and Horgan, Dan and Kroiss, Manuel and Danihelka, Ivo and Dudzik, Andrew and Chung, Junyoung and Choi, David H. and Powell, Richard and Ewalds, Timo and Huang, Aja and Sifre, Laurent and Cai, Trevor and Agapiou, John P. and Jaderberg, Max and Vezhnevets, Alexander S. and Leblond, R{\'e}mi and Pohlen, Tobias and Dalibard, Valentin and Budden, David and Sulsky, Yury and Molloy, James and Paine, Tom L. and Gulcehre, Caglar and Wang, Ziyu and Pfaff, Tobias and Wu, Yuhuai and Ring, Roman and Yogatama, Dani and W{\"u}nsch, Dario and McKinney, Katrina and Smith, Oliver and Schaul, Tom and Lillicrap, Timothy and Kavukcuoglu, Koray and Hassabis, Demis and Apps, Chris and Silver, David},
  journal = {Nature},
  volume  = {575},
  number  = {7782},
  pages   = {350--354},
  year    = {2019},
  doi     = {10.1038/s41586-019-1724-z}
}

@book{hutter2019Automated,
  title = {Automated {{Machine Learning}}: {{Methods}}, {{Systems}}, {{Challenges}}},
  shorttitle = {Automated {{Machine Learning}}},
  editor = {Hutter, Frank and Kotthoff, Lars and Vanschoren, Joaquin},
  year = {2019},
  series = {The {{Springer Series}} on {{Challenges}} in {{Machine Learning}}},
  publisher = {Springer International Publishing},
  location = {Cham},
  doi = {10.1007/978-3-030-05318-5},
  url = {http://link.springer.com/10.1007/978-3-030-05318-5},
  urldate = {2025-11-12},
  isbn = {978-3-030-05317-8 978-3-030-05318-5},
  langid = {english}
}

@article{li2018Hyperband,
  title   = {Hyperband: A Novel Bandit-Based Approach to Hyperparameter Optimization},
  author  = {Li, Lisha and Jamieson, Kevin and DeSalvo, Giulia and Rostamizadeh, Afshin and Talwalkar, Ameet},
  journal = {Journal of Machine Learning Research},
  volume  = {18},
  number  = {185},
  pages   = {1--52},
  year    = {2018}
}

@article{archambeau2024Hyperparameter,
  title = {Hyperparameter {{Optimization}} in {{Machine Learning}}},
  author = {Luca Franceschi and Michele Donini and Valerio Perrone and Aaron Klein and Cédric Archambeau and Matthias Seeger and Massimiliano Pontil and Paolo Frasconi},
  year = {2025},
journal={Foundations and Trends in Machine Learning},
volume={18},
number={6}, 
pages={1054-1201},
  langid = {english}
}

@article{dalibard2021Faster,
  title = {Faster {{Improvement Rate Population Based Training}}},
  author = {Dalibard, Valentin and Jaderberg, Max},
  date = {2021-09-28},
  eprint = {2109.13800},
  journal = {Preprint arXiv:2109.13800},
  eprintclass = {cs},
  doi = {10.48550/arXiv.2109.13800},
  url = {http://arxiv.org/abs/2109.13800},
  year = {2021},
  pubstate = {prepublished}
}

@article{lee2025Evolving,
  title   = {Evolving Deeper LLM Thinking},
  author  = {Lee, Kuang-Huei and Fischer, Ian and Wu, Yueh-Hua and Marwood, Dave and Baluja, Shumeet and Schuurmans, Dale and Chen, Xinyun},
  journal = {Preprint arXiv:2501.09891},
  year    = {2025}
}

@article{delmoral2000branching,
     author = {Del Moral, Pierre and Miclo, Laurent},
     title = {Branching and interacting particle systems. {Approximations} of {Feynman-Kac} formulae with applications to non-linear filtering},
     journal = {S\'eminaire de probabilit\'es},
     pages = {1--145},
     year = {2000},
     publisher = {Springer - Lecture Notes in Mathematics},
     volume = {34},
     mrnumber = {1768060},
     zbl = {0963.60040},
     language = {en},
     url = {https://www.numdam.org/item/SPS_2000__34__1_0/}
}

@incollection{delmoral2001Asymptotic,
  title = {Asymptotic {{Results}} for {{Genetic Algorithms}} with {{Applications}} to {{Nonlinear Estimation}}},
  booktitle = {Theoretical {{Aspects}} of {{Evolutionary Computing}}},
  author = {Del Moral, Pierre and Miclo, Laurent},
  editor = {Kallel, Leila and Naudts, Bart and Rogers, Alex},
  editora = {Rozenberg, G. and family=Bäck, given=Th., given-i={{Th}} and Eiben, A. E. and Kok, J. N. and Spaink, H. P.},
  editoratype = {redactor},
  year = {2001},
  pages = {439--493},
  publisher = {Springer Berlin Heidelberg},
  location = {Berlin, Heidelberg},
  doi = {10.1007/978-3-662-04448-3_22},
  url = {http://link.springer.com/10.1007/978-3-662-04448-3_22},
  urldate = {2025-11-27},
  isbn = {978-3-642-08676-2 978-3-662-04448-3}
}

@article{delmoral2001stability,
  title = {On the Stability of Interacting Processes with Applications to Filtering and Genetic Algorithms},
  author = {Del Moral, Pierre},
  year = {2001},
  journal = {Annales de l'Institut Henri Poincare (B) Probability and Statistics},
  volume = {37},
  number = {2},
  pages = {155--194},
  issn = {02460203},
  doi = {10.1016/S0246-0203(00)01064-5},
  url = {https://linkinghub.elsevier.com/retrieve/pii/S0246020300010645},
  urldate = {2025-11-27}
}

\appendix

\section{Notation}
\label{app:notation}

Random variables are assumed to be defined over an underlying probability space $(\Omega, \hat{\mathcal{F}}, \mathbb{P})$. With $X\sim f$, we indicate that the law of a random variable $X$ is $f$. For a Borel set $A$, $\unif(A)$ is the uniform probability measure over $A$, while $\mathcal{N}(m,\Sigma)$ is a normal distribution with mean $m$ and covariance matrix $\Sigma$. With $\delta_x$ we indicate the Dirac delta probability measure centred in $x$, while the Bernoulli distribution is $\textup{Bern}(\tau) = (1-\tau)\delta_0 + \tau \delta_1$, for $\tau \in [0,1]$. For a right continuous curve $x(t), t\in [0,T]$ we indicate $x(t^-) := \lim_{s \to t, s\leq 0}x(s)$, and $x(t^+)$ analogously. With $\bm{1}_d  = (1, \dots, 1)^\top$ and $\bm{0}_d = (0, \dots, 0)^\top$ we denote $d$-dimensional column vectors made of 1s and 0s, respectively. The minimum and maximum operation are defined as $a\wedge b: = \min(a,b)$, and $a \vee b = \max(a,b)$. For a test function $\phi:\RR^d \to \RR$, $\|\phi\|_{\mathrm{Lip}}: = \max_{x \neq y}|\phi(x)- \phi(y)|/|x - y|$, and $\|\phi\|_\infty :=\sup_{x}|\phi(x)|$. With $B(x,R)$ we indicate the open ball or radius $R$ centred at $x$, and $\chi_A$ is the indicator function for a a set $A$. The discrete probability simplex is defined ad $\Delta_N = \{ \w\in \RR^N_{\geq 0}\,|\, \sum_{i=1}^N w^i = 1 \}$.

With $\mathcal{P}(\RR^d)$ we indicate the set of Borel probability measures over $\RR^d$, and with $\mathcal{P}_p(\RR^d)$ those with bounded $p$-th moment, $M_p(f): = \int |x|^p f(dx) <\infty$. We will sometime use the shorter notation 
\[
\langle \phi, f\rangle := \int \phi(x) f(d x)
\]
to indicate integrals for a measure $f$ and a measurable test function $\phi$.
We equip $\mathcal{P}_p(\RR^d)$ with the $L^p$-Wasserstein distance \citep{villani2009}
\[
\W_p(f,g) := \left(\mathbb{E}_{X\sim f, Y\sim g}|X - Y|^p \right)^{1/p}\,.
\]

\section{Proof of Theorem \ref{t:chaos}: propagation of chaos}
\label{app:chaos}

The proof follows the technique proposed in \citep{borghi2026chaos} to study the many agent limit of genetic algorithm. Here, we need to additionally consider the SGD training in the parameters space.

For notational simplicity, we will use the stack parameters $\theta\in \RR^{d_\theta}$, and hyperparameters $h\in \RR^{d_h}$ in single vector $x := (\theta, h)\in \RR^d$ with $d := d_\theta + d_h$.  
Note that Algorithm \ref{alg:pbt} divides the evolution of the $N$ neural networks agents $X^i = (\theta^i, h^i)$ into intervals of length $\tau\in (0,1]$. 
Inside the intervals, agents evolve according to the gradient-based training dynamics, while at $t = t_n$, they are subject to a genetic PBT update.
Let $t_n = n\tau$, and consider the function $\bm{b}(\cdot): \RR^d \to \RR^d$ and matrices $\bm{\Sigma}, \bm{\sigma}\in \RR^{d\times d}$
\[
\bm{b}(x) := 
\begin{pmatrix} 
- \nabla_\theta \LL(\theta, h) \\
\bm{0}_{d_h}
\end{pmatrix}\,,\quad  
\bm{\Sigma}:= 2\beta^{-1} \mathrm{diag}
\begin{pmatrix} 
\bm{1}_{d_\theta} \\
\bm{0}_{d_h}
\end{pmatrix} \quad 
\textup{and}
\quad \bm{\sigma}:=\sigma \mathrm{diag} \begin{pmatrix} 
\bm{0}_{d_\theta} \\
\bm{1}_{d_h}
\end{pmatrix}\,.
\]
For a given probability measure $f \in \mathcal{P}(\RR^d)$, we write the selection operator for the fitness as 
\[
G_\F[f](dx) : = \frac{e^{\F(x)} f(dx)}{\langle e^{\F}, f\rangle}\,.
\]

The evolution for $(X_t^i)_{t\geq0}$ described by Algorithm \ref{alg:pbt}, can then be re-written as following. Starting from some initial data $X_0^i\sim f_0$, for each iteration $n = 0,1,\dots$, in the open interval $(t_{n}, t_{n+1})$ the agent is defined as the solution to the Cauchy problem
\begin{equation} \label{eq:training}
\begin{dcases}
dX_t^i = \bm{b}(X_t^i)dt + \bm{\Sigma}^{1/2} dW^i_t \qquad t\in (t_n, t_{n+1}) \\
\lim_{t\to t_{n}^+}X_t^i = X_{t_n}^i 
\end{dcases}    \qquad i = 1,\dots, N\,,
\end{equation}
while at the discrete times $t_{n+1}$, agents undergo a resampling-mutation procedure. Recall that an agent $j$ is selected to be copied with a probability $w^j  = \exp(\F(X^j))/\sum_\ell \exp(\F(X^\ell))$. One can implement it as follows: given a set of weights $\w\in \Delta_N$, define the index maps $\mathbf{j}(\w.\cdot):[0,N)\to \{1, \dots, N\}$ as
\[
\mathbf{j}(\w, \alpha): = \min\left\{j \in \{1,\dots,N\}\,\middle |\, \alpha < N\sum_{\ell = 1}^j w^\ell \right\}\,.
\]
Note that if $\alpha \sim [0,N)$, then $\mathbb{P}(\mathbf{j}(\w, \alpha) = j) = w^j$. Consider i.i.d. random variables $\tau^i_n \sim \textup{Bern}(\tau)$, $\alpha^i\sim \unif[0,N)$, $\xi_n^i\sim \mathcal{N}(\bm{0}_d, I_d)$, and let $\w_n: = \exp(\F(X_{t_{n+1}^-}^j))/\sum_\ell \exp(\F(X_{t_{n+1}^-}^j)$. The genetic update can then be written as
\begin{equation}\label{eq:resampling}
X^i_{t_{n+1}} = (1 -\tau_n^i)X^i_{t_{n+1}^-} + \tau_n^i \left( X^{\mathbf{j}(\w_n, \alpha^i_n)}_{t_{n+1}^-} + \bm{\sigma} \xi^i_n \right)\,.
\end{equation}

Recall the limiting kinetic PDE model \eqref{eq:fpde} is given by 
\begin{equation} \label{eq:fpde2}
\begin{dcases}
\frac{\partial f(t)}{\partial t}  = T_\LL[f(t)] + E_\F[f(t)] \qquad t \in [0,t_{\max}] \\
\lim_{t \to 0^+} f(t) = f_0\,,
\end{dcases}
\end{equation}
where the training and evolutionary operator are weakly defined, in our notation, as
\begin{align*}
\langle \phi, T_\LL[f]\rangle &:= \int\left(  \phi(x) \cdot \bm{b}(x) + \frac12 \nabla\cdot (\bm{\Sigma} \nabla \phi)  \right) f(dx)  \\
\langle \phi, E_\F[f]\rangle &:= \int  \phi(\tilde{x} + \bm{\sigma}\xi)\, G_\F[f](d\tilde{x}) \,\mu^\xi(d\xi)  - \int \phi(x) f(dx) \,.
\end{align*}

\begin{definition}[Weak measure solution] \label{def:sol} Let $f_0 \in \mathcal{P}_q(\RR^d)$ be an initial datum. We say that $f\in C([0,T], \mathcal{P}_q(\RR^d))$ is a weak measure solution to the Cauchy problem \eqref{eq:fpde2} if for  any $\phi\in C_0^\infty(\RR^d)$ it holds
\[
\frac{d}{dt}\langle \phi, f(t) \rangle = \langle \phi, T_\LL[f(t)] + \langle \phi, E_\F[f(t)]\rangle
\]
for almost every $t\in [0,t_{\max}]$, and $\lim_{t\to 0^+}\langle \phi, f(t)\rangle = \langle \phi, f_0 \rangle$.  
\end{definition}

\subsection{Preliminaries}

We introduce an equivalent formulation of the BL norm \eqref{eq:BL}, as an optimal transport problem. Let $f, g\in \mathcal{P}(\RR^d)$. Thanks to the Kantorovich--Rubinstein duality \cite[Theorem 1.14]{villani2009}, it holds
\begin{equation}\label{eq:wasswedge}
    \| f - g\|_{\mathrm{BL}} = \inf_{X\sim, Y\sim g} \mathbb{E}\left[|X - Y|\wedge 1 \right]\,.
\end{equation}
That is, the Wasserstein-1 distance with cost $\D(x,y) := |X - Y|\wedge 1$ admits a dual representation which corresponds to the BL norm, see \citep{hairer2008spectral,piccoli2016properties,borghi2026chaos} for more details. An interesting consequence is that, since $\D(x,y)\leq |x - y|$, $\|\cdot \|_{\mathrm{BL}}$ is upper bounded by the usual Wasserstein-1 distance $\W_1$ with cost $|x - y|$, 
\[
\| f - g\|_{\mathrm{BL}} \leq \W_1(f,g)\,.
\]

We collect some auxiliary results that will be useful later.

\begin{lemma}[Concentration rate] \label{l:rate}
   Let $f\in \mathcal{P}_q(\RR^d)$ with $q >1$ as in Assumption \ref{asm:chaos}, and $f^N = (1/N)\sum_{i=1}^N\delta_{X^i}$, with $X^i \sim f$ i.i.d.. There exists a constant $C = C(d,q)>0$ such that for all $N\geq 1$:
   \[
\| f - f^N\|_{\mathrm{BL}} \leq C M_q^{1/q}(f)\, \epsilon(N)
   \]
   with $\epsilon(N) \to 0$ as $N\to \infty$.
\end{lemma}
\begin{proof}
The inequality holds for $\W_1(f,f^N)$ as a special case of \citep[Theorem 1]{fournier2015rate}. Then, we use that $\|f - f^N\|_{\mathrm{BL}}\leq \W_1(f,f^N)$. The explicit form of the rate $\epsilon(N)$ depends on $q, d$, and can be found in \citep{fournier2015rate}.
\end{proof}

\begin{lemma}[Selection stability {\citep[Lemma 2.4]{borghi2026chaos}}]\label{l:selection}
Let $\F$ be bounded, globally Lipschitz and strictly positive. There exists a positive constant $C = C(\F)$ such that for any $f,g\in \mathcal{P}(\RR^d)$ it holds
\[
\| G_\F[f] - G_\F[g]\|_{{\mathrm{BL}}} \leq C_\F\|f - g\|_{\mathrm{BL}}\,.
\]
\end{lemma}

\subsection{Limit $N \to \infty$}

We aim to prove propagation of chaos via coupling methods, where the system of interacting agents $\mathbf{X}_t: = (X_t^1, \dots, X_t^N)$ is coupled with a system of non-linear, non-interacting agents $\overline{\mathbf{X}}^i_t: = (\overline{X}_t^1, \dots, \overline{X}_t^N)$. Then, we aim to estimate their distance by using the optimal transport formulation of the BL norm \eqref{eq:wasswedge}.

\textit{Non-linear particle system.}
Initial data is the same: $\mathbf{X}_0 = \overline{\mathbf{X}}_0$. Then, as before, the non-linear system evolves in the time interval $(t_n,t_{n+1})$ as
\begin{equation*} 
\begin{dcases}
d\overline{X}_t^i = \bm{b}(\overline{X}_t^i)dt + \bm{\Sigma}^{1/2} dW^i_t \qquad t\in (t_n, t_{n+1}) \\
\lim_{t\to t_{n}^+}\overline{X}_t^i = \overline{X}_{t_n}^i 
\end{dcases}    \quad i = 1,\dots, N\,,
\end{equation*}
with the same Brownian paths $(W_t^i)_{t\geq 0}$ of the system $\mathbf{X}_t$. 

Let $f^\tau(t): = \textup{Law}(\overline{X}_t^i)$. To define the genetic PBT update, we consider a measurable function $X^*_n: (\RR^d)^N\times \textup{int}(\Delta_N)\times [0,N)\to \RR^d$ with the following property: if $\alpha \sim \unif[0,N)$, then
\[
\textup{Law}\left (X^*\left(\w_n, \mathbf{X}_{t_{n+1}^-},\alpha\right)\right) = G_\F\left[f(t_{n+1}^{-}) \right]\,,
\]
and
\[
\mathbb{E}_\alpha\left[\left | X_n^*\left(\w_n, \mathbf{X}_{t_{n+1}^-},\alpha\right) -  X^{\mathbf{j}(\w_m, \alpha_n^i)}\right| \wedge 1  \right] = \left\|G_\F \left[f^\tau(t_{n+1}^-)\right] - G_\F\left[f^{N,\tau}(t_{n+1}^-)\right]\right\|_{{\mathrm{BL}}}\, .
\]
Existence of such map is ensured by \citep[Lemma 3.1]{borghi2026chaos} under Assumption \ref{asm:chaos}. 

The genetic update for each non-linear particle is then given by
\begin{equation}\label{eq:resampling}
\overline{X}^i_{t_{n+1}} = (1 -\tau_n^i)\overline{X}^i_{t_{n+1}^-} + \tau_n^i \left( X_n^*(\mathbf{X}_t,\w_n, \alpha^i_n) + \bm{\sigma} \xi^i_n \right)\,,
\end{equation}
where we used the same random variables $\alpha^i_n, \tau_n^i, \xi_n^i$ of \eqref{eq:resampling}. The evolution of $\overline{X}_t^i$ is of McKean--Vlasov type since it depends on the law $f^\tau(t)$ of the agent itself. Also, the agents are interdependent on one another since the auxiliary random variables $\alpha^i_n, \tau_n^i, \xi_n^i$ are all sampled independently.

\begin{lemma}[Moments bound] \label{l:moments} Assume $f_0 \in \mathcal{P}_q(\RR^d)$, and Assumption \ref{asm:chaos} holds for $\LL,\F$. Then
$f^\tau(t) = \textup{Law}(\overline{X}_t^i)$ satisfies
\[
M_q^{1/q}(f^\tau(t)) \leq C \left (1 + M_q^{1/q}(f_0) \right) e^{C't }\qquad \textup{for all}\quad t \in [0,t_{\max}]\,,
\]
for some positive constants $C,C'$ depending on $\LL,\F,\beta, \sigma, q$.
\end{lemma}
\begin{proof} By Assumption \ref{asm:chaos}, $\LL$ is globally Lipschitz and so the drift coefficient in the SDE \eqref{eq:training} is Lipschitz and admits linear growth. By standards results in SDE theory, see for instance \citep[Theorem 4.1]{mao2011sde}, it holds for $q\geq 2$
\[
\mathbb{E}|\overline{X}_t^i|^q \leq C' (1 + \mathbb{E}|\overline{X}_{t_n}^i|^q)e^{C (t - t_{n})}\quad \textup{for all}\quad t\in [t_n, t_{n+1})
\]
with $C = C(\LL, \beta, q)$, $C' = C'(q)>0$. This leads to $M_q(f^\tau(t^-_{n+1})) \leq C'(1 + M_1(f^\tau(t_{n})))e^{C \tau}$. Since by Assumption \ref{asm:chaos} the fitness function $\F$ satisfies $0< \inf\F \leq \F(x)\leq \sup \F< \infty$ for all $x\in \RR^d$, we have for some constant $\tilde C$ depending on $\sigma,q$
\begin{align*}
\int |x|^qf^\tau(t_{n+1}, dx) &  = (1 - \tau)M_q(f^\tau(t_{n+1}^-)) + \tau  \int |x + \bm{\sigma}\xi|^q G_\F[f^\tau(t_{n+1}^-)](dx) \\
& \leq(1 - \tau)M_q(f^\tau(t_{n+1}^-)) + \tau \tilde C \left (1 +  \int \frac{e^{\F(x)} |x|^q}{\langle e^{\F}, f^\tau(t_{n+1}^-)\rangle} f^\tau(t_{n+1}^-, dx)   \right) \\
& \leq (1 - \tau)M_q(f^\tau(t_{n+1}^-)) + \tau \tilde C \left ( 1 + \exp(\sup{\F} - \inf{\F}) M_q(f^\tau (t_{n+1}^-))\right) \,.
\end{align*}
Therefore, we obtain the estimate for some constants $C,C'$ depending on $\LL, \F, \sigma, \beta,  q>0$
\[
M_q(f^\tau(t_{n+1})) \leq (1 + C\tau)\max\left \{ M_1(f^\tau(t_{n+1}^-)),1\right\} \leq e^{C'\tau} \max\{M_q(f^\tau(t_{n})),1\}
\]
where we further used that $1 + s \leq e^s$. The desired estimate can be obtain by iterating the argument until $t_n = t_{\max}$.    
\end{proof}

\textit{Gr\"{o}nwall-type argument.}
We now consider  the empirical measure  $\overline{f}^{N, \tau}(t)$ associated with the non linear system $\overline{\textbf{X}}_t$, and track its distance from $f^{N,\tau}(t)$. Recall that since $\overline{f}^{N,\tau}(t)$ is made of independent i.i.d particles with law $f^\tau(t)$, and so  $\|\overline{f}^{N\tau}(t) - f^\tau(t)\|_{\mathrm{BL}}\lesssim \epsilon (N)$ by Lemma \ref{l:rate}.

Since the Brownian paths are shared among the two systems, we have for $t \in [t_n, t_{n+1})$
\begin{align*}
d|X_t^i - \overline{X}_t^i |^2 & = 2(X_t^i - \overline{X}_t^i)\cdot (\bm{b}(X_t^i) - \bm{b}(\overline{X}_t^i))dt \\
& \leq 2 \|\LL\|_{\textup{Lip}}|X_t^i - \overline{X}_t^i |^2 dt \,.
\end{align*}
By Gr\"onwall's inequality, we obtain $|X_t^i - \overline{X}_t^i|^2 \leq \exp(2\|\LL\|_{\textup{Lip}}(t - t_{n}))|X_{t_n}^i - \overline{X}_t^i|^2$ from which also follows 
\begin{equation}\label{eq:est:a}
\left(|X^i_t - \overline{X}^i_{t}| \wedge 1\right) \leq e^{\|\LL\|_{\textup{Lip}}(t - t_{n})}\left(|X^i_{t_n} - \overline{X}^i_{t_n}| \wedge 1\right) \qquad \textup{for}\quad t\in [t_n, t_{n+1})\,.
\end{equation}
Next, by definition of $X_{t_{n+1}}^i, \overline{X}_{t_{n+1}}^i$, and in particularly their optimal coupling with respect to $\D(x,y) = |x- y|\wedge 1$ we have
\begin{align*}
\mathbb{E}\left(|X_{t_{n+1}}^i - \overline{X}_{t_{n+1}}^i|\wedge 1 \right) & =
(1 - \tau) \mathbb{E}\left(|X_{t_{n+1}^-}^i - \overline{X}_{t_{n+1}^-}^i|\wedge 1 \right) \\
 &\qquad + \tau \mathbb{E}\left( |X_n^*\left(\w_n, \mathbf{X}_{t_{n+1}^-},\alpha\right) -  X^{\mathbf{j}(\w_m, \alpha_n^i)}|\wedge 1   \right) \\
 & = (1 - \tau) \mathbb{E}\left(|X_{t_{n+1}^-}^i - \overline{X}_{t_{n+1}^-}^i|\wedge 1 \right) \\
 &\qquad + \tau \left \| G_\F\left[f^{N,\tau}(t_{n+1}^-)\right] - G_\F\left[ f^\tau(t_{n+1}^-) \right]\right \|_{\mathrm{BL}}\,\\
 & \leq (1 - \tau) \mathbb{E}\left(|X_{t_{n+1}^-}^i - \overline{X}_{t_{n+1}^-}^i|\wedge 1 \right) 
 + \tau  C_\F \left \| f^{N,\tau}(t_{n+1}^-) -  f^\tau(t_{n+1}^-) \right \|_{\mathrm{BL}}\,
\end{align*}
by Lemma \ref{l:selection}.
To bound the second term, we apply triangular inequality together with Lemmas \ref{l:rate} and \ref{l:moments}, and obtain 
\begin{align*}
\| f^{N,\tau}(t_{n+1}^-) - f^{\tau}(t_{n+1}^-)\|_{\mathrm{BL}} &
\leq  \| f^{N,\tau}(t_{n+1}^-) - \overline{f}^{N,\tau}(t_{n+1}^-)\|_{\mathrm{BL}} + 
\| \overline{f}^{N,\tau}(t_{n+1}^-) - f^{\tau}(t_{n+1}^-)\|_{\mathrm{BL}} \\
& \leq \frac1N\sum_{i=1}^N \left(|X_{t_{n+1}^-}^i - \overline{X}_{t_{n+1}^-}^i| \wedge 1 \right)  
+ C(1 + M^{1/q}_q(f_0))e^{C' t_{n+1} } \epsilon(N)\,.
\end{align*}
Summing for all $i = 1,\dots, N$ and applying \eqref{eq:est:a}, we get the estimate
\begin{align*}
    E_{n+1}: &= \frac1N\sum_{i=1}^N \mathbb{E}\left(|X_{t_{n+1}}^i - \overline{X}_{t_{n+1}}^i|\wedge 1 \right) \\
    & \leq  (1- C \tau) e^{\|\LL\|_{\textup{Lip}} \tau}  E_n  + \tau C(1 + M^{1/q}_q(f_0))\epsilon(N) e^{C't_{n+1}} \\
    & \leq e^{C \tau} E_n + \tau C(1 + M^{1/q}_q(f_0))\epsilon(N) e^{C't_{n+1}}\,.
\end{align*}
By iterating for all $n$, until $t_n = t_{\max}$, one obtains since $E_0 = 0$
\begin{align*}
E_n & \leq e^{C t_n}E_0 + \tau C(1 + M^{1/q}_q(f_0)) \sum_{k = 0}^{n-1} e^{C't_n}(1 + C\tau)^n \\
& \leq C t_n (1 + M^{1/q}_q(f_0))e^{C't_n} \epsilon(N)\,.
\end{align*}
where the constants depend on $\LL,\F,\beta, \sigma$, and $q$. By the optimal transport formulation of the BL norm \eqref{eq:wasswedge}, we finally get 
\begin{equation} \label{eq:limN}
\sup_{t \in [0, t_{\max}]}\|f^{N,\tau}(t) - f^\tau(t)\|_{{\mathrm{BL}}}\;\; \lesssim \;\;\epsilon(N)\,.
\end{equation}
where the hidden constant depends also on $t_{\max}$, but not on $\tau$.

\subsection{Limit $\tau \to 0$}

We study the limit towards the time-continuous model by first showing that $(f^\tau)_{\tau>0}$ admits a converging subsequence, and then by identifying the limit with a weak measure solution to \eqref{eq:fpde2}. In the following, 
$(\overline{X}_t^\tau)_{t\in [0,t_{\max}]}$ is an an arbitrary non-linear particle (defined in the previous section), where the apex $\tau$ stresses the dependence on the time-discretisation parameter. We have $ \textup{Law}(\overline{X}_t^\tau) = f^\tau(t)$.

Note that $\overline{X}^\tau, \tau>0,$ are not continuous processes due to the jump at times $t_n = n\tau$, but still they are right-continuous and the left limit always exists.
Therefore, we consider the topology given by the Skorokhod space of c\'adl\'ag functions, and show that the sequence is compact and tight. This approach is standard, and we refer to the surveys \citep{diez2022chaos1,diez2022chaos2} for a reminder on Skorokhod topology in the context chaos propagation results. We outline the main arguments below. 

First, we note that the moments bound derived in Lemma \ref{l:mass} is independent of $\tau$. Therefore, the $q$-th moments of $f^\tau(t)$ are uniformly bounded for all $t\in [0,t_{\max}]$ and $\tau>0$. 
Then, while $f^\tau$ is not continuous due to the jumps, the jump size is determined by the parameter $\tau$. 
Also, we have for every
\[
\mathbb{E}|\overline{X}^\tau_t - \overline{X}^\tau_s|\lesssim  |s - t| + \sqrt{|t - s|}\qquad \textup{for any}\quad [s, t]\subset [t_n, t_{n+1})\,,
\]
thanks to the globally Lipschitz continuity of the drift $\bm{b}$ and the constant diffusion coefficients $\bm{\Sigma}^{1/2}$. 
Therefore any sequence $(\tau_k)_{k\in \mathbb{N}}$ such that $\tau_k \to 0$ leads to a sequence $(\overline{X}^{\tau_k}_{t})_{t\in [0,t_{\max}]}$, $k\in \mathbb{N}$, which satisfies Assumption 12, 13, and 14 in \citep{diez2022chaos2}. By \citep[Theorem C.5]{diez2022chaos2}, then, we have existence of a limit $f\in C([0,t_{\max}], \mathcal{P}(\RR^d))$.

Fix any $\tau>0$ and take $\phi \in C_c^\infty(\RR^d)$. It holds
\begin{align*}
\langle \phi, f^\tau(t_{n+1}) - f^\tau(t_n)\rangle  & =  
\langle \phi, f^\tau(t_{n+1}) - f^\tau(t_{n+1}^-)\rangle + 
\langle \phi, f^\tau(t^-_{n+1}) - f^\tau(t_n)\rangle \\
& = \tau \left \langle \phi, E_\F\left[f^\tau(t_{n+1}^-)\right]\right \rangle
+ \int_{t_n}^{t_{n+1}} \left(\langle \nabla \phi\cdot \bm{b}, f(s) \rangle
+ \frac12
\langle \nabla\cdot \bm{\Sigma}\nabla \phi, f(s)\rangle  
\right)ds \\
& =   \int_{t_n}^{t_{n+1}} \left(\langle \phi , E_\F [f^\tau(s)] \rangle
+ 
\langle \phi, T_\LL[f^\tau(s)]\rangle  
\right)ds \\
&\qquad+\int_{t_{n}}^{t_{n+1}}\langle \phi, E\left[f^\tau(t_{n+1}^-)\right] - E\left[f^\tau(s)\right] \rangle ds\,.
\end{align*}
The first term corresponds exactly to the PDE operator, while the second is an error term coming from the fact that the evolutionary operator acts only at discrete times $t_n$ in the algorithm.

Next, thanks to the Lipschitz continuity of the evolutionary operator and of $f^\tau(s), s\in (t_n, t_{n+1})$, we have
\[
\int_{t_{n}}^{t_{n+1}}| \langle \phi, E\left[f^\tau(t_{n+1}^-)\right] - E\left[f^\tau(s)\right] \rangle | ds
\leq C\int_{t_n}^{t_{n+1}} \|f^\tau(t_{n+1}^-) - f^\tau(s) \|_{\mathrm{BL}} ds \leq C' \tau^2\,.
\]
Altogether, one obtains the estimate
\begin{align*}
\left |\langle \phi, f^\tau(t) - f^\tau(0) \rangle 
- \int_0^t  \left(\langle \phi , E_\F [f^\tau(s)] \rangle
+ 
\langle \phi, T_\LL[f^\tau(s)]\rangle  
\right) ds\right|
  \\ \leq  \sum_{n =0}^{\lfloor t/\tau\rfloor + 1}\int_{t_n}^{t_{n+1}\wedge t}
 \langle \phi, E\left[f^\tau(t_{n+1}^-)\right] - E\left[f^\tau(s)\right] \rangle ds 
  \leq C' t \, \tau 
\end{align*}
Then, by taking the limit $\tau \to 0$, one obtains for all $t\in [0, t_{\max}]$ 
\[
\langle \phi, f(t) - f(0) \rangle 
 =  \int_0^t  \left(\langle \phi , E_\F [f(s)] \rangle
+ 
\langle \phi, T_\LL[f(s)]\rangle  
\right) ds\,,
\]
and so $f\in C([0,t_{\max}], \mathcal{P}_q(\RR^d))$ is a weak measure solution to the Cauchy problem \eqref{eq:fpde2}.

\section{Proof of Theorem 4.1: convergence to the fittest solution}

\subsection{Proof of Theorem \ref{t:convergence}}
\label{app:laplace}

We divide the proof in different steps. First, thanks to a bound on the solution's energy, we show that the mass around the maximizer $h^\star$ does not vanish over the computational interval $[0,t_{\max}]$. Then, we apply the quantitative  Laplace principle \eqref{eq:laplace} and the mean evolution \eqref{eq:meanevo} to show that it converges towards the fittest point $h^\star$.

\textit{Evolution of second moments.}
We denote the energy, or second moment, of $\rho$ as $\textup{En}(\rho) =M_2(\rho) = \int |h|^2 \rho(dh)$. Thanks to Assumption \ref{asm:fitness} and \eqref{eq:energy}, we have
\[
M_2\left(G_{\alpha \overline{\F}}[\rho]\right) \leq b_1 + b_2 M_2(\rho)
\]
where $b_1, b_2>0$ can be made independent of the selection pressure, by simply assuming, for instance, $\alpha>1$. Since the evolutionary operator consists of selection and convolution with a Gaussian kernel, we also have $M_2(E_{\alpha\overline{\F}}(\rho)) \leq C( 1 + M_2(\rho))$ with $C = C(\sigma, b_1, b_2)>0$ a constant independent on $\alpha$. Similarly to Lemma \ref{l:moments}, one can obtain an estimates over the the entire time horizon 
\[
M_2(\rho(t)) \leq C(1 + M_2(\rho_0)) e^{C't}\qquad t\in [0, t_{\max}]
\]
for a solution $\rho \in C([0,T], \mathcal{P}_2(\RR^{d_h}))$ to the averaged PDE \eqref{eq:rhopde}. We omit the details for brevity. This allows us to infer that there exists $\overline{M}>0$ such that 
\begin{equation}\label{eq:unifmom}
\sup_{t\in [0,t_{\max}]}M_2(G_{\alpha\overline{\F}}[\rho(t)])\leq \overline{M} \qquad \textup{for all} \quad \alpha >1\,.   
\end{equation}

\begin{lemma}[Mass around maximizer]\label{l:mass}
Let $r>0$ be fixed and $\overline{\F}$ satisfy Assumption \ref{asm:fitness}. For any given measure $\rho\in \mathcal{P}_2(\RR^{d_h})$ with $M_2(G_{\alpha \overline{\F}}[\rho])\leq \overline{M}$, there exists $\delta = \delta(r, \overline{M})>0$ such that for 
\[
\mu: = E_{\alpha\overline{\F}}[\rho] + \rho\qquad \textup{it holds} \qquad \mu(B(h^\star, r))\geq \delta \,.
\]
\end{lemma}
\begin{proof}
Let $R> \sqrt{\overline{M}}$, and such that $|h^\star|< R$. By Markov's inequality, there exists $\delta_R \in (0,1)$ such that
\[
G_{\alpha\overline{\F}}[\rho](B(0,R)) \geq 1 - \frac{M_2\left(G_{\alpha\overline{\F}}[\rho]\right)}{R^2} \geq 1 - \delta_R\,.
\]
Next, recall that after selection we perform mutation by convolution with a Gaussian measure:
$\mu  = K_\sigma * G_{\alpha \overline{\F}}[\rho]$\,.
Since the Gaussian kernel is not degenerate, there exists a density lower bound $\underline{K}_R>0$ such that for all $h\in B(0,R)$
\[
\int \chi_{B(h^\star, r)}(h +\sigma \xi) K_\sigma(d \xi) \geq |B(h^\star, r)|\underline{K}_R\,,
\]
where $\chi_A(\cdot)$ is the indicator function for a set $A$. 
It follows
\begin{align*}
\mu(B(h^\star, r)) &= \iint \chi_{B(h^\star, r)}(h +\sigma \xi) K_\sigma(d \xi) G_{\alpha\overline{\F}}[\rho](dh) \\
&  \geq \int_{\RR^{d_h}} \int_{B(0,R)} \chi_{B(h^\star, r)}(h +\sigma \xi) K_\sigma(d \xi) G_{\alpha\overline{\F}}[\rho](dh) \\
&\geq  |B(h^\star, r)|\underline{K}_R G_{\alpha \overline{\F}}[\rho](B(0,R)) \\
&\geq  |B(h^\star, r)|\underline{K}_R \delta_R>0\,.
\end{align*}
\end{proof}

\textit{Evolution of mass over  $B(h^\star, r)$.}
Thanks to \eqref{eq:unifmom}, we can apply Lemma \ref{l:mass} to the PDE solution $\rho(t)$, and obtain that for any $r>0$ there exists $\delta>0 $
such that 
\begin{align*}
E_{\alpha \overline{\F}}[\rho(t)](B(h^\star,r)) &  = 
\left( E_{\alpha \overline{\F}}[\rho(t)] + \rho(t) \right) (B(h^\star,r))  - \rho(t) (B(h^\star,r)) \\
&\geq \delta  - \rho(t, B(h^\star,r))\,,
\end{align*}
leading to the estimate 
\[
\frac{d}{dt} \rho(t, B(h^\star,r)) \geq \delta  - \rho(t, B(h^\star,r))\,.
\]
By Gr\"onwall's inequality, it holds
\[
\rho(t, B(h^\star,r)) \geq  e^{-t}\rho_0(B(h^\star, r)) + (1 - e^{-t}) \delta \qquad \textup{for}\quad t \in [0,t_{\max}]\,.
\]
In particular, the mass over $B(h^\star,r)$ is strictly positive during the entire interval $[0,t_{\max}]$

\textit{Quantitative Laplace principle.}
In the following, we denote the weighted mean as  $\mup^\alpha[\rho]: = \mup\left(G_{\alpha \overline{\F}} [\rho(t)]\right)$ for simplicity.
Let $tol>0$ be fixed. By the quantitative Laplace principle \eqref{eq:laplace}, for any $\rho\in \mathcal{P}_2(\RR^{d_h})$, under the inverse continuity Assumption \ref{asm:fitness} (iii) we have that there exists $r_{tol}>0$ and $q_{tol}>0$ sufficiently small such that 
\begin{equation*}
|\mup^\alpha[\rho(t)] - h^\star| \leq \sqrt{\frac{{tol}}8} + \frac{\exp(-\alpha q_{tol})}{\rho(t, B(h^\star, r_{tol}))}\int |h - h^*| \rho(dh)\,.
\end{equation*}
Now, we apply the estimate for our solution $\rho(t)$, together with the previous estimates for the mass around the fittest solution, and energy bounds. In particular, we use that $M_2(\rho(t))\leq \overline{M}  = C(1 + M_2(\rho_0))e^{Ct_{\max}}$ for $t\in [0,t_{\max}]$, and Lemma \ref{l:mass} with $r_{tol}$ and $\overline{M}$. We get
\begin{align*}
|\mup^\alpha[\rho(t)] - h^\star| &\leq \sqrt{\frac{{tol}}8} + \frac{\exp(-\alpha q_{tol})}{\rho(t, B(h^\star, r_{tol}))}\int |h - h^*| \rho(t,dh)\\
& \leq \sqrt{\frac{{tol}}8}  + \frac{\exp(-\alpha q_{tol})}{\rho_0(B(h^\star, r_{tol}))\wedge  \delta }2\left(|h^\star| + \sqrt{\overline{M}}\right)\,.
\end{align*}
Since $q_{tol}>0$, this estimate allows as to take $\alpha\gg 1$ sufficiently large such that the second term is of order $\sqrt{{tol}}$, and obtain 
\[
|\mup^\alpha[\rho(t)] - h^\star| \leq \sqrt{\frac{{tol}}8} + \sqrt{\frac{{tol}}8}  = \sqrt{\frac{{tol}}2}  \qquad \textup{for all} \quad t\in [0,t_{\max}]\,.
\]
By plugging this in the mean evolution, we get
\begin{align*}
\frac{d}{dt}|\mup[\rho(t)] - h^\star|^2 & = 
2(\mup[\rho(t)] - h^\star)\cdot (\mup^\alpha[\rho(t)] - \mup[\rho(t)]) \\
& = - 2|\mup[\rho(t)] - h^*|^2 + 2(\mup[\rho(t)] - h^*)\cdot (\mup^\alpha[\rho(t)] - \mup[\rho(t)]) \\
& \leq - |\mup[\rho(t)] - h^*|^2 + |\mup^\alpha[\rho(t)] - h^\star|^2\\
& \leq - |\mup[\rho(t)] - h^*|^2 + \frac{{tol}}2\,.
\end{align*}
By Gr\"ownall's inequality, we obtain the desired upper bound
\[
|\mup[\rho(t)] - h^\star|^2 \leq e^{-t}|\mup[\rho_0] - h^\star|^2 + (1 - e^{-t})\frac{{tol}}2\,.
\]
By this estimate, one can directly see that if $t_{\max}$ is sufficiently large, as $t^\star$ in Theorem \ref{t:convergence}, then both terms in the right-hand-side can be made of order $tol/2$.

\subsection{Additional proofs}
\label{app:addproofs}

\begin{proof}(Lemma \ref{l:penalization})
Fix $h\in H$. 
Under Assumption~\ref{asm:laplace}, the function
\[
\theta \longmapsto \LL(\theta,h)-\LL(\theta^\star(h),h)
\]
satisfies the assumptions of \citep[Assumption 3.1]{hasenpflug2024wasserstein} with $p=1$ and reference density $\pi_0\equiv 1$. 
The particular case where the solution is unique is discussed in \citep[Remark 3.10]{hasenpflug2024wasserstein}.
Recall $\W_1$ denotes the $L^1$-Wasserstein distance. By  \citep[Theorem 3.8]{hasenpflug2024wasserstein}
\[
\W_1\bigl(\mu^\infty_\LL(\cdot|h),\delta_{\theta^\star(h)}\bigr)
\le C \frac{1}{\sqrt{\beta}}.
\]
The constant can be be independent on $h$ thanks to the compactness of the hyperparameter space $H$ and the uniform estimate on the $\nabla^2\LL(\theta^\star(h),h)$.
Next, by Assumption~\ref{asm:chaos}, the fitness $\F$ is bounded and globally Lipschitz on $\RR^{d_\theta}\times H$, and so is $e^{\F(\theta,h)}$.
By Kantorovich--Rubinstein duality \cite[Theorem 1.14]{villani2009}, we have for some constant $C_1>0$ independent of $h$
\begin{align*}
\left|\int e^{\F(\theta,h)}\,\mu^\infty_\LL(d\theta|h)-e^{\F(\theta^\star(h),h)}\right|
&=
\left|\int e^{\F(\theta,h)}\,\mu^\infty_\LL(d\theta|h)-\int e^{\F(\theta,h)}\,\delta_{\theta^\star(h)}(d\theta)\right| \\
&\le \|e^{\F(\cdot,h)}\|_{\textup{Lip}}\,
\W_1\bigl(\mu^\infty_\LL(\cdot|h),\delta_{\theta^\star(h)}\bigr) \\
&\le C_1\frac{1}{\sqrt{\beta}}.
\end{align*}

By Assumption \ref{asm:chaos}, we have $\inf\F>0$ and so 
$
\int e^{\F(\theta,h)}\,\mu^\infty_\LL(d\theta|h)\ge e^{\inf\F}$
and
$
e^{\F(\theta^\star(h),h)}\ge e^{\inf\F}.
$
Since the $\log$ function is globally Lipschitz on $[e^{{\inf\F}},\infty)$, 
we obtain
\begin{align*}
|\overline{\F}_\beta(h)-\F(\theta^\star(h),h)|
&=
\left|
\log\int e^{\F(\theta,h)}\,\mu^\infty_\LL(d\theta|h)
-\log e^{\F(\theta^\star(h),h)}
\right| \\
&\le e^{-{\inf\F}}
\left|
\int e^{\F(\theta,h)}\,\mu^\infty_\LL(d\theta|h)
-e^{\F(\theta^\star(h),h)}
\right| \\
&\le C_2\frac{1}{\sqrt{\beta}},
\end{align*}
for some constant $C_2>0$ independent of $h$.
\end{proof}

\section{Details of the Reinforcement Learning experiment}
\label{app:numerics}

In value-based reinforcement learning, an approach to maximise the expected cumulative reward is \emph{Q-learning}
\citep{watkins1992Qlearning}. It aims to estimate the optimal
action--value function $Q^\ast$, defined as
\begin{equation}
    Q^\ast(s,a)
    =
    \mathbb{E}\!\left[
        \sum_{k=0}^{\infty} \gamma^{k} R_{t+k+1}
        \,\middle|\,
        s_t = s,\; a_t = a
    \right],
\end{equation}
where $\gamma \in (0,1)$ is a discount factor controlling the relative
importance of future rewards.

The optimal action--value function satisfies the \emph{Bellman optimality
equation}
\begin{equation}
    Q^\ast(s,a)
    =
    \mathbb{E}_{s' \sim \mathbb{P}(\cdot \mid s,a)}
    \!\left[
        R(s,a)
        + \gamma \max_{a'} Q^\ast(s',a')
    \right],
\end{equation}
which characterises $Q^\ast$ as the unique fixed point of the Bellman
optimality operator.

Q-learning approximates this fixed point via an iterative, sample-based
procedure, replacing the expectation in the Bellman equation with
single-step transitions observed from the environment. The algorithm
proceeds as follows:
\begin{enumerate}
    \item At each time step, select an action according to a policy
    derived from the current estimate $Q$. A commonly used exploration
    strategy is the $p$-greedy policy: with probability
    $p$, a random action is selected (exploration), and with
    probability $1-p$ the greedy action
    \begin{equation}
        a(s) = \operatorname*{arg\,max}_{a'} Q(s,a')
    \end{equation}
    is chosen (exploitation). The exploration rate $p$ is
    typically annealed from an initial value
    $p_{\mathrm{start}}$ to a final value
    $p_{\mathrm{end}}$, for example using an exponential
    schedule:
    \begin{equation}\label{eq:edecay}
        p_{\mathrm{next}}
        =
        p_{\mathrm{end}}
        +
        (p_{\mathrm{start}} - p_{\mathrm{end}})
        \exp\!\left(
            -\frac{\mathrm{total\ steps}}{p_{\mathrm{decay}}}
        \right).
    \end{equation}
    We set $p_{\mathrm{start}}=1, p_{\mathrm{end}}=0.01$, and usually it starts with larger values of $p_{\mathrm{decay}}$ and follows an annealing schedule, so that in the beginning of the training it explores then exploits.

    \item Execute the selected action and observe the reward $r$ and the
    next state $s'$.

    \item Update the action--value estimate according to
    \begin{equation}\label{eq:q-learning}
        Q(s,a)
        \leftarrow
        Q(s,a)+\alpha \Bigl[ r+ \gamma \max_{a} Q(s',a) - Q(s,a) \Bigr],
    \end{equation}
    where $\alpha \in (0,1]$ is a learning rate, typically set to be decreased over
    time.

    \item Set the current state to $s'$ and repeat until a terminal
    condition is reached.
\end{enumerate}

Under suitable conditions on exploration and learning rates, Q-learning
converges to $Q^\ast$ for finite Markov decision processes.

For problems with large or continuous state spaces, tabular Q-learning becomes infeasible due to the need to store
$Q(s,a)$ for all state--action pairs. This motivates the use of function
approximation, leading to the \emph{Deep Q-Network} (DQN) algorithm
\citep{mnih2015Humanlevel}. In DQN, a neural network parameterised by
$\theta$ is used to approximate the action--value function,
\[
    Q(s,a;\theta) \approx Q^\ast(s,a).
\]

The network parameters are optimised by minimising the mean-squared
temporal-difference error. At iteration $i$, the loss function is given
by
\begin{equation}
    L_i(\theta_i)=\mathbb{E}_{(s,a,r,s') \sim U(D)}\left[ \left( r+\gamma \max_{a'} Q(s',a';\theta_i^-)-Q(s,a;\theta_i)\right)^2 \right],
\end{equation}
where $\theta_i^-$ denotes the parameters of a \emph{target network},
which is held fixed for several iterations to stabilise training and discount factor $\gamma$ was set to $0.99$ throughout. The
expectation is approximated by sampling uniformly from a replay buffer
$D = \{ e_1, \dots, e_t \}$, where each experience
$e_t = (s_t, a_t, r_t, s_{t+1})$ is stored during interaction with the
environment. This mechanism, known as \emph{experience replay}, reduces temporal correlations between samples and improves data efficiency, enabling stable training of deep value-based reinforcement learning algorithms.

We chose three hyperparameters, namely, learning rate of the Adam optimiser, the $p$-decay parameter $p_{\mathrm{decay}}$ from \eqref{eq:edecay}, and the experience replay batch size. Their respective ranges were set to $(10^{-5}, 10^{-2})$, $(500, 5000)$, and $(32, 128)$. During mutation, we scaled these domains to $(-1, 1)$ before adding noise with the noise scale $\sigma=0.1$. This differs from the original PBT method and the Ray library, which typically perturb values by a multiplicative factor of $1.2$ or $0.8$.

\end{document}